\documentclass[sn-mathphys,Numbered]{sn-jnl}

\usepackage{graphicx}%
\usepackage{multirow}%
\usepackage{amsmath,amssymb,amsfonts}%
\usepackage{amsthm}%
\usepackage{comment}
\usepackage{mathrsfs}%
\usepackage[title]{appendix}%
\usepackage{xcolor}%
\usepackage{textcomp}%
\usepackage{manyfoot}%
\usepackage{booktabs}%
\usepackage{algorithm}%
\usepackage{algorithmicx}%
\usepackage{algpseudocode}%
\usepackage{listings}%
\usepackage[utf8]{inputenc}
\usepackage{mathptmx}
\usepackage{centernot}

\usepackage[inline]{enumitem}
\usepackage{tabularx}
\usepackage{hyperref}
\usepackage{tikz}
\usetikzlibrary{trees,shapes,decorations,calc,arrows.meta,positioning}
\usepackage{subfig}
\usepackage{url}
\usepackage{thmtools}
\usepackage{thm-restate}

\newtheorem{innercustomthm}{Proposition}
\newenvironment{customprop}[1]
  {\renewcommand\theinnercustomthm{#1}\innercustomthm}
  {\endinnercustomthm}

\newtheorem{example}{Example}
\newtheorem{definition}{Definition}

\def\arga{\ensuremath{\mathsf{a}}}
\def\argb{\ensuremath{\mathsf{b}}}
\def\argc{\ensuremath{\mathsf{c}}}
\def\argd{\ensuremath{\mathsf{d}}}
\def\arge{\ensuremath{\mathsf{e}}}
\def\argf{\ensuremath{\mathsf{f}}}

\def\aga{\ensuremath{{\cal A}_0}}
\def\agb{\ensuremath{{\cal A}_1}}
\def\agc{\ensuremath{{\cal A}_2}}
\def\agi{\ensuremath{{\cal A}_i}}
\def\agj{\ensuremath{{\cal A}_j}}

\begin{document}
\title[Disagree and Commit: Degrees of Argumentation-based Agreements]{Disagree and Commit: Degrees of Argumentation-based Agreements\footnote{This paper is based on preliminary results of an extended abstract originally presented at the 20th Conference on Autonomous Agents and MultiAgent Systems, AAMAS'21 \cite{DBLP:conf/atal/KampikN21}.}}

\author*[1]{\fnm{Timotheus} \sur{Kampik}}\email{tkampik@cs.umu.se}
\author[1]{\fnm{Juan Carlos} \sur{Nieves}}\email{jcnieves@cs.umu.se}
\affil*[1]{\orgdiv{Department of Computing Science}, \orgname{Umeå University}, \orgaddress{\city{Umeå}, \country{Sweden}}}

\abstract{
In cooperative human decision-making, agreements are often not total; a partial degree of agreement is sufficient to commit to a decision and move on, as long as one is somewhat confident that the involved parties are likely to stand by their commitment in the future, given no drastic unexpected changes.
In this paper, we introduce the notion of \emph{agreement scenarios} that allow artificial autonomous agents to reach such agreements, using formal models of argumentation, in particular abstract argumentation and value-based argumentation.
We introduce the notions of degrees of satisfaction and (minimum, mean, and median) agreement, as well as a measure of the impact a value in a value-based argumentation framework has on these notions. We then analyze how degrees of agreement are affected when agreement scenarios are expanded with new information, to shed light on the reliability of partial agreements in dynamic scenarios. An implementation of the introduced concepts is provided as part of an argumentation-based reasoning software library.}

\keywords{
formal argumentation, dialogues, agreement technologies, group decision-making
}

\maketitle


\section{Introduction}
\label{intro}
In Artificial Intelligence (AI) research, devising formal models and algorithms that specify how autonomous agents can reach agreements is an important research direction~\cite{ossowski2012agreement}.
In this context, the symbolic AI community considers \emph{formal argumentation} approaches~\cite{BENCHCAPON2007619,baroni2018handbook} as particularly promising.
Recently, such approaches have, for example, been proposed to align the moral values of different stakeholders in decision automation scenarios~\cite{liao2019building} and to resolve rule conflicts in medical decision-support systems~\cite{chunli2018}.
From a more generic perspective, recent research has introduced a formal approach to determining \emph{degrees of agreement} in formal argumentation dialogues, in which agents add arguments on a specific topic to a knowledge base~\cite{nieves2019approximating}. The intuition behind this notion is that for practical purposes, it is often not necessary (or possible) to reach a complete agreement; instead, agents may decide that a certain degree of agreement on a given topic is sufficient to commit to roughly aligned decisions and move on.
In management practice, this approach is sometimes referred to as ``\emph{disagree and commit}''~\cite{PATANAKUL2009216}, emphasizing that while discourse is vital, at some point stakeholders will have to align in order to lay the prerequisites for successful execution.
In this paper, we introduce this notion to formal argumentation, in particular to abstract and value-based argumentation, and work towards answering the following research questions about multi-agent agreements in abstract and value-based argumentation:
\begin{enumerate}
    \item How can a set of agents determine to what degree they are agreeing on a topic (set of arguments)?
    \item How do an agent's subjective value preferences impact the degree of agreement on a topic?
    \item How are degrees of agreement affected when agreement scenarios are expanded with new information?
\end{enumerate}
To answer Question 1 and 2, we introduce a formal framework for \emph{agreement scenarios} and \emph{degrees of satisfaction and agreement} to abstract argumentation~\cite{dung1995acceptability}, as well as to value-based argumentation~\cite{bench2003persuasion}. To answer Question 3, we apply and extend formal properties that are systematic relaxations of monotony of entailment and conduct a basic empirical analysis using synthetic data.

Let us introduce examples to illustrate the contribution this paper makes to the research questions.
First, we take a step back, introducing a simple \emph{choice-based} agreement scenario.

\begin{example}[Degrees of Agreement in Simple Choice Scenarios]\label{example:intro}
We have three agents ($\aga, \agb, \agc$), who are C-level managers and discuss which strategic initiatives among $\arga$, $\argb$, and $\argc$ are the most important ones.
Considering the complex socio-professional nature of the problem, reaching full consensus on all questions is intractable.
As long as everyone roughly agrees on the importance, the managers will be content and assume that their objectives are aligned to a sufficient degree.
Table~\ref{table:1}
shows \emph{ranks} and \emph{degrees} of satisfaction of the managers given the different choice options, assuming the agents have established a total preorder of preferences on the powerset of the set of all options\footnote{A total preorder on a set $S$ is a binary relation $\succeq$ on $S$, s.t. for all $x,y,z \in S$, \emph{i)} $x \succeq x$ (reflexivity); \emph{ii)} $x \succeq y$ and $y \succeq z$ imply $x \succeq z$ (transitivity); \emph{iii)} $x \succeq y$ or $y \succeq x$ (totality).}.
Here, we assume that the option of Rank 1 in Table~\ref{table:1} is an agent's most preferred option and the ranks of all other options are inferred from this option.
Table~\ref{table:1} assumes that the agents care about agreement with respect to the inclusion as well as exclusion of options: here, we may assume that the ranking is based on a similarity measure between the most preferred and other options. For example, we may measure similarity between $\{\argb, \argc\}$ and $\{\argb\}$ by computing the number of joint options in plus the number of joint options out, divided by the total number of options, i.e. $\frac{|\{\argb, \argc\} \cap \{\argb\}| + | \{\arga\} \cap \{\arga, \argc\}|}{|\{\arga, \argb, \argc\}|} = \frac{2}{3}$, and between $\{\argb, \argc\}$ and $\{\arga, \argc\}$ by computing $\frac{|\{\argb, \argc\} \cap \{\arga, \argc\}| + | \{\arga\} \cap \{\argb\}|}{|\{\arga, \argb, \argc\}|} = \frac{1}{3}$. Hence, given the most preferred option $\{\argb, \argc\}$, $\agb$ ranks $\{\argb\}$ higher than $\{\arga, \argc\}$.

\begin{table*}
\caption{Preference ranks and degrees of satisfaction (\emph{deg.} in parentheses) of agents $\aga$, $\agb$, and $\agc$ as total preorders on $2^{\{\arga, \argb, \argc\}}$, assuming one most preferred set of options and a similarity measure that is \emph{sensitive} to options that are jointly not inferred.}
\label{table:1}
\renewcommand{\arraystretch}{1.0}
\centering
\begin{tabular}{ |c|c|c|c|} 
 \hline
 Rank (deg.) & $\aga$ & $\agb$  & $\agc$\\ 
  \hline
 $1 (1)$ & $\{\arga, \argb, \argc\}$ & $\{\argb, \argc\}$ & $\{\}$ \\ 
 $2 (\frac{2}{3})$ & $\{\arga, \argb\}$ or $\{\arga, \argc\}$ or $\{\argb, \argc\}$ & $\{\arga, \argb, \argc\}$ or $\{\argb\}$ or $\{\argc\}$ & $\{\arga\}$ or $\{\argb\}$ or $\{\argc\}$ \\ 
 $3 (\frac{1}{3})$ & $\{\arga\}$ or $\{\argb\}$ or $\{\argc\}$& $\{\}$ or $\{\arga, \argb\}$ or $\{\arga, \argc\}$ & $\{\arga, \argb\}$ or $\{\arga, \argc\}$ or $\{\argb, \argc\}$ \\
 $4 (0)$ & $\{\}$ & $\{\arga\}$ & $ \{\arga, \argb, \argc\}$ \\
 \hline
\end{tabular}
\end{table*}

To determine the degree of satisfaction of an agent $\agi$ with another agent's $\agj, i,j \in \{0, 1, 2\}$, position, we determine the maximal similarity of any most preferred option of $\agi$ and any most preferred option of $\agj$ (see Table~\ref{table:2}).

\begin{table*}
\caption{Matrix: one agent's rank and degree of satisfaction (the latter in parentheses) given another agent's choice option, considering preferences from Table~\ref{table:1}.}
\label{table:2}
\centering
\renewcommand{\arraystretch}{1.0}
\begin{tabular}{ |c|c|c|c|} 
 \hline
 \phantom{ } & $\aga$ & $\agb$  & $\agc$\\ 
  \hline
 $\aga$ & $1 (1)$ & $2 (\frac{2}{3})$ & $4 (0)$ \\ 
 $\agb$ & $2 (\frac{2}{3})$ & $1 (1)$ & $3 (\frac{1}{3})$ \\
 $\agc$ & $4 (0)$ & $3 (\frac{1}{3})$ & $1 (1)$ \\ 
 \hline
\end{tabular}
\end{table*}
To determine the degree of agreement between the whole group of agents, we introduce the following approaches:
\begin{itemize}
    \item The \textbf{degree of minimal agreement} is the lowest of the degrees of satisfaction of all agents given an option that yields the maximal lowest degree of satisfaction among all agents. In the example scenario the degree of minimal agreement is $\frac{1}{3}$, e.g., provided by option $\{\arga, \argc\}$; in our example all options that provide a degree of satisfaction of greater than $0$ to all agents, i.e., all options except $\{\}$, $\{\arga\}$, and $\{\arga, \argb, \argc\}$ provide the degree of minimal agreement.
    \item The \textbf{degree of mean agreement} is the mean of the degrees of satisfaction of all agents, given an option that allows for a maximal mean of the degrees of satisfaction among all agents. The degree of mean agreement is $\frac{2}{3}$: the option $\{\argb, \argc\}$ provides the degrees of satisfaction $\frac{2}{3}$ to $\aga$, $1$ to $\agb$, and $\frac{1}{3}$ to $\agc$, averaging at $\frac{2}{3}$.
    \item Similarly, the \textbf{degree of median agreement} of the example is $\frac{2}{3}$, e.g., the median of $\langle 1, \frac{2}{3}, 0 \rangle$, given the option $\{\arga, \argb, \argc\}$ (other options, such as $\{\argb, \argc\}$ or $\{\argb\}$ would work as well).
\end{itemize}
\end{example}

The degrees of agreement can then, for instance, inform decisions on whether to further deliberate a given topic---in the example, the strategic initiatives---or guide future decisions of the involved participants; in our case, the lack of management alignment as indicated by Tables~\ref{table:1}~and~\ref{table:2} 
should cause each manager to be careful when making any future strategy-related decision.

Another aspect that can inform future decisions is how reliably the agents will keep their opinions given some constraints.
This requires the analysis of the agents' decision processes, either by means of observation or---in particular in the case of artificial agents/computer systems---by formal verification.
A straight-forward approach is to simply check whether the preferences of an agent are consistent over time, a property emerging from economic rationality\footnote{Let us note that our agents are not economically rational, given that the emerging preferences from extension-based abstract argumentation are not necessarily consistent, as shown in~\cite{10.1093/logcom/exab003}.}; in our context, we consider specific \emph{principles} that a function governing an agent's decision-making may satisfy.
To also account for the \emph{knowledge} the agents use to establish their preferences, we apply abstract argumentation~\cite{dung1995acceptability} to model the agents' inference processes.
In the context of our agreement problem, we consider the choice items a subset of the \emph{arguments} (atomic items) of an abstract argumentation framework; we call this subset of arguments the \emph{topic}. Based on an argumentation framework's arguments and their attack relation (binary relation on the arguments), an \emph{argumentation semantics} determines which sets of arguments can be considered valid conclusions; these sets of arguments are called \emph{extensions}.
Given an argumentation-based model of an agreement scenario, we can impose formal constraints on argumentation semantics that allow us to guarantee that---under specific conditions---the degree of agreement between a group of agents remains within specific bounds as new arguments are added to an argumentation framework.
Let us extend the previous example to illustrate what an argumentation-based agreement scenario is.
\begin{example}[Degrees of Agreement in Argumentation Scenarios]\label{example-2}
Let us assume that the agents have jointly constructed the argumentation framework as depicted in Figure~\ref{fig:example-1}, but they use different inference functions (argumentation semantics) to reach their conclusions (to determine extensions of arguments).
Note that the ``self-attacking'' arguments $\argd$ and $\arge$---in our case chosen to highlight differences between semantics in a very simple example---may model knowledge that is self-contradictory, but still attempts to defeat some of the topic arguments.
\begin{figure}[!ht]
        \centering
        \begin{tikzpicture}[
            scale=0.7, every node/.style={scale=0.7},
            noanode/.style={dashed, circle, draw=black!60, minimum size=10mm, font=\bfseries},
            unanode/.style={circle, draw=black!75, minimum size=10mm, font=\bfseries},
            anode/.style={circle, fill=lightgray, draw=black!75, minimum size=10mm, font=\bfseries},
            ]
            \node[anode]    (a)    at(0,4)  {\arga};
            \node[anode]    (b)    at(2,4)  {\argb};
            \node[anode]    (c)    at(4,4)  {\argc};
            \node[noanode]    (d)    at(0,2)  {\argd};
             \node[noanode]    (e)    at(2,2)  {\arge};
            \path [->, line width=0.5mm]  (d) edge node[left] {} (a);
            \path [->, line width=0.5mm]  (e) edge node[left] {} (b);
            \path [->, line width=0.5mm]  (e) edge node[left] {} (c);
            \path [->, line width=0.5mm]  (b) edge node[left] {} (e);
            \path [->, line width=0.5mm]  (c) edge node[left] {} (e);
            \draw[->, line width=0.5mm] (d.-90) arc (180:180+264:4mm);
            \draw[->, line width=0.5mm] (e.-90) arc (180:180+264:4mm);
        \end{tikzpicture}
\caption{Abstract argumentation framework (Example~\ref{example-2}). Here and henceforth, arguments in gray are in all extensions (here, assuming stage semantics); arguments with dashed border are in no extension. Arguments in white with a solid border would be in some, but not in all extensions (cf. Figure~\ref{fig:example-2}) or indicate that no semantics has been applied to infer extensions (cf. Figure~\ref{fig:example-exp}).}
\label{fig:example-1}
\end{figure}
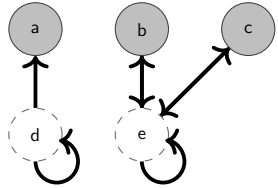
If the agents were to use the following argumentation semantics, they would reach the conclusions (i.e., infer exactly the extensions) as presented by the highest ranked options in Table~\ref{table:1}: $\aga$: stage semantics~\cite{verheij1996two}; $\agb$: preferred semantics~\cite{dung1995acceptability}; $\agc$: grounded semantics~\cite{dung1995acceptability} (to be formally introduced in Section~\ref{prelim}).
To reflect the ranks and degrees of satisfaction in Table~\ref{table:1}, an agent can determine their preferences using measures of similarity between any set of choice options (let us call them \emph{topic arguments} in the context of formal argumentation) and the most similar topic arguments returned by the agent's argumentation semantics, as formalized in Section~\ref{sec:argumentation}.
\end{example}
In an argumentation scenario, the agents can then make informed decisions on how reliable an agreement is, based on formal argumentation principles that are relaxed forms of monotony and ensure the following properties when \emph{normally expanding} an argumentation framework (adding new arguments without changing the relationships between existing arguments):
\begin{enumerate}
    \item \emph{Weak cautious monotony:} if no new argument attacks a specific extension of the original argumentation framework, every argument in this extension is also in an extension of the argumentation framework's normal expansion~\cite{10.1093/logcom/exab003}.
    \item \emph{Strong relaxed monotony:} if no \emph{unattacked} argument directly or indirectly (without ``interruption'') attacks a specific extension of the original argumentation framework, every argument in this extension is also in an extension of the argumentation framework's normal expansion. This paper introduces the principle to demonstrate that agents may commit to enforcing rather strict principles even if this implies the violation of the behavior of the semantics they originally employ, intuitively to ensure that inferences remain more steady (change less) in dynamic scenarios and to better align the way inferences are drawn across agents.
    \item In addition, we introduce an abstract class of principles, which we call \emph{relaxed monotony principles} (and which includes the aforementioned principles), for which we show that an upper bound of change w.r.t. the degree of minimal agreement can be guaranteed, given that the relaxed monotony condition is not infringed.
\end{enumerate}
For example, we can see that preferred semantics does not satisfy the strong relaxed monotony principle; adding a new argument $\argf$ to the argumentation framework in Figure~\ref{fig:example-1}, such that $\argf$ attacks itself, as well as $\argb$ and $\argc$ would cause agent $\agb$ to consider only $\{\}$ an extension of the new argumentation framework.
To also support semantics that do not satisfy a specific relaxed monotony principle, we introduce an approach that allows agents to commit to a principle in disregard of the properties of the semantics they apply.

Let us introduce another example to illustrate that a similar approach can be used to determine degrees of satisfaction and agreement in the context of value-based argumentation~\cite{bench2003persuasion}, where abstract argumentation frameworks are extended by values associated with arguments and agents' subjective value preferences. %
In the abstract argumentation example above, the differences between the inferred extensions arise mostly due to nuances in the behaviors of the semantics that are applied. In contrast, in value-based argumentation we can lift such differences to the knowledge modeling level, which yields more applicable (albeit technically slightly less straightforward) perspectives.
While we remain on an ``abstract'' level in the example below, for which we do not provide a real-world interpretation, value-based argumentation typically leads to precisely the multi-agent disagreements that we cover in this paper and is, for example, applied\footnote{Let us note that we are referring to somewhat ``academic'' applications here and not to large-scale applications in real-world software systems.} to democratic decision support~\cite{DBLP:conf/comma/Atkinson06} and regulatory compliance~\cite{DBLP:journals/ail/BurgemeestreHT11}. 
\begin{example}[Degrees of Agreement in Value-based Argumentation Scenarios]\label{example-3}
We start with the argumentation framework $AF = (\{\arga, \argb, \argc,  \argd\}, \{(\arga, \argb), (\argb, \arga), (\argc, \argb), (\argd, \argc)\})$ as depicted in Figure~\ref{fig:example-2}.
\begin{figure}[!ht]
        \centering
        \begin{tikzpicture}[
            scale=0.7, every node/.style={scale=0.7},
            noanode/.style={dashed, circle, draw=black!60, minimum size=10mm, font=\bfseries},
            unanode/.style={circle, draw=black!75, minimum size=10mm, font=\bfseries},
            anode/.style={circle, fill=lightgray, draw=black!75, minimum size=10mm, font=\bfseries},
            ]
            \node[unanode]    (a)    at(0,4)  {\arga};
            \node[unanode]    (b)    at(2,4)  {\argb};
            \node[noanode]    (c)    at(2,2)  {\argc};
            \node[anode]    (d)    at(0,2)  {\argd};
            \path [->, line width=0.5mm]  (a) edge node[left] {} (b);
            \path [->, line width=0.5mm]  (b) edge node[left] {} (a);
            \path [->, line width=0.5mm]  (c) edge node[left] {} (b);
            \path [->, line width=0.5mm]  (d) edge node[left] {} (c);
        \end{tikzpicture}
\caption{Abstract argumentation framework (Example~\ref{example-3}).}
\label{fig:example-2}
\end{figure}
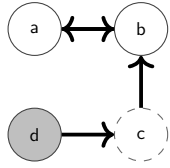
Each argument in $AF$ is mapped to a value in $\{ a_v, b_v, c_v, d_v \}$, for the sake of simplicity as follows: argument $\arga$ to value $a_v$, $\argb$ to $b_v$, $\argc$ to $c_v$, and $\argd$ to $d_v$.
Each of three agents ($\aga$, $\agb$, $\agc$) has established an additional binary relation (\emph{value preferences}) on the values in $AF$; intuitively, the abstract argumentation framework and argument-to-value-mapping models the objective facts, whereas the value preferences model subjective beliefs, ultimately about the effective strength of the arguments.
The agents have the following value preferences:
\begin{itemize}
    \item $\aga$: $a_v$ is preferred over $b_v$;
    \item $\agb$: $b_v$ is preferred over $a_v$;
    \item $\agc$: $c_v$ is preferred over $d_v$.
\end{itemize}
When interpreting these preferences in accordance with value-based argumentation (i.e., removing, given an agent's preferences, all attacks in which the attacked argument's value is preferred over the attacking argument's value\footnote{Alternative approaches to \emph{preference-based argumentation} exist that take a more nuanced approach to handling attacks in face of conflicting preferences~\cite{kaci2018preference,AMGOUD2014585,10.1093/logcom/exac094}. Preference-based argumentation is a special case of value-based argumentation; still, we may claim the these approaches can be relatively straightforwardly adjusted to cover value-based argumentation and hence be applied to our value-based argumentation-based agreement scenarios. However, we abstain from going into further detail in order to avoid scope creep.}), we get the following \emph{subjective} argumentation frameworks $AF_i, 0 \leq i \leq 2$ and extensions $ES_i$ (assuming preferred semantics) for each agent $\agi$:
\begin{itemize}
    \item $AF_0 = (\{\arga, \argb, \argc,  \argd\}, \{(\arga, \argb), (\argc, \argb), (\argd, \argc)\})$, $ES_0 = \{\{\arga, \argd\}\}$;
    \item $AF_1 = (\{\arga, \argb, \argc,  \argd\}, \{(\argb, \arga), (\argc, \argb), (\argd, \argc) \})$, $ES_1 = \{\{\argb, \argd\}\}$;
    \item $AF_2 = (\{\arga, \argb, \argc,  \argd\}, \{(\arga, \argb), (\argb, \arga), (\argc, \argb)\})$, $ES_2 = \{\{\arga, \argc, \argd\}\}$.
\end{itemize}
This allows us to again determine the degrees of satisfaction, as well as the degrees of agreement between the agents, using the same approaches we have defined for abstract argumentation. Here, we are interested in all arguments, i.e., our topic is $\{\arga, \argb, \argc, \argd\}$. The degrees of minimal, mean, and median agreement are $\frac{1}{2}$, $\frac{3}{4}$ and $\frac{3}{4}$, respectively, using the similarity measure sketched in Example~\ref{example:intro}.
Moreover, we can determine the \emph{impact} that a value has on the degrees of satisfaction and agreement in a given scenario, by determining the delta between the actual degree and the counterfactual, assuming the given value was not present.
For example, in our scenario the impact of the value $b_v$ on the degree of minimal agreement is $-\frac{1}{4}$, as ``removing'' the value $b_v$ from our scenario increases the degree of minimal agreement from $\frac{1}{2}$ to $\frac{3}{4}$.
\end{example}
To guarantee limits in the change of degrees of agreement when expanding a value-based argumentation framework in agreement scenarios under some constraints, we can rely on the weak cautious monotony principle.
Let us highlight that in value-based argumentation-based agreement scenarios, the different value preferences may represent different agents, but they may also model different \emph{agent-internal} perspectives as suggested by some theories of human cognition, such as the theory of planned behavior~\cite{AJZEN1991179}.

The rest of this paper is organized as follows.
Section~\ref{prelim} introduces the theoretical foundations of our work.
Then, Section~\ref{sec:argumentation} provides a formal model of argumentation-based agreement scenarios and degrees of agreements in abstract argumentation, which is extended for expansion-based dialogues and augmented by a formal analysis in Section~\ref{exp-aas}.
Analogously, the framework is extended to value-based argumentation in Section~\ref{vba} and expansion-based dialogues for value-based argumentation are covered in Section~\ref{exp-vaas}.
A software implementation of our approach\footnote{The implementation of the formal concepts that we introduce is available as part of the DiArg argumentation-based dialogue reasoner~\cite{kampik2020diarg} on GitHub at \url{http://s.cs.umu.se/mhfrcp}.} is presented in Section~\ref{sec:implementation}, alongside with initial experiments that shed some light on the theoretical parts of our work from an empirical perspective.
Finally, Section~\ref{related} discusses our results in the context of related research before Section~\ref{questions} concludes the paper by highlighting future research potential.

\section{Theoretical Preliminaries}
\label{prelim}
Let us introduce the formal preliminaries that are relevant in the context of this paper.
An (abstract) argumentation framework is a tuple $AF = (AR, AT)$, where $AR$ is a set of elements (called \emph{arguments}) from our background argument set $\cal AR$, and $AT$ is a binary relation (called \emph{attacks}) on $AR$ (i.e., $AR \subseteq AR \times AT$)~\cite{dung1995acceptability}. We denote the class of all argumentation frameworks by ${\cal F}$. We assume that the set of arguments in an argumentation framework is finite. For $\arga, \argb \in AR$ such that $(\arga, \argb) \in AT$, we say that ``$\arga$ attacks $\argb$''.
Given $S \subseteq AR$, we define $S^+$ as $\{\argc | \argc \in AR, \exists \argd \in S$, such that $\argd$ attacks $\argc\}$; given $\arge \in AR$, we say that ``$S$ attacks $\arge$'' iff $\exists \arge' \in S$, such that $\arge'$ attacks $\arge$.
$a \in AR$ is \emph{acceptable} w.r.t. $S$ iff for each $\argb \in AR$ it holds true that if $\argb$ attacks $\arga$, then $\argb$ is attacked by $S$.
$S \subseteq AR$ is:
\begin{itemize}
    \item \emph{conflict-free} iff $\nexists \arga, \argb \in S$, such that $\arga$ attacks $\argb$;
    \item \emph{admissible} iff $S$ is conflict-free and each argument in $S$ is acceptable w.r.t. $S$.
\end{itemize}
We say that a set $S \subseteq AR$ \emph{strongly defends} an argument $\arga \in AR$ iff $\forall \argb \in AR$, such that $\argb$ attacks $\arga$, $\exists \argc \in S \setminus \{\arga\}$, such that $\argc$ attacks $\argb$ and $\argc$ is strongly defended by $S \setminus \{\arga\}$.
An (argumentation) semantics $\sigma: {\cal F} \rightarrow 2^{2^{\cal AR}}$ maps an argumentation framework $AF = (AR, AT)$ to $ES \subseteq 2^{AR}$, where every set of arguments $E \in ES$ is called a \emph{$\sigma$-extension} of AF.
Informally, we may say that a semantics (or, indirectly, an agent that uses a semantics) infers one or several extensions from an argumentation framework, or given an extension, we may say that a semantics has inferred the arguments that are in this extension.
$\sigma(AF)$ denotes all $\sigma$-extensions of AF and ${\cal S}$ denotes the class of all argumentation semantics.
Some classical argumentation semantics as introduced by Dung~\cite{dung1995acceptability} are the \emph{complete}, \emph{preferred}, and \emph{grounded} semantics.
\begin{definition}[Admissible Set-based Argumentation Semantics~\cite{dung1995acceptability}]
Given an argumentation framework $AF = (AR, AT)$, a set $S \subseteq AR$ is:
\begin{itemize}
    \item a \emph{complete extension} iff $S$ is admissible and each argument that is acceptable w.r.t. $S$ belongs to $S$. $\sigma_{co}(AF)$ denotes all complete extensions of $AF$;
    \item a \emph{preferred extension} of $AF$ iff $S$ is a maximal (w.r.t. set inclusion) admissible subset of $AR$. $\sigma_{pr}(AF)$ denotes all preferred extensions of $AF$;
    \item a \emph{grounded extension} of $AF$ iff $S$ is the minimal (w.r.t. set inclusion) complete extension of $AF$. $\sigma_{gr}(AF)$ denotes all grounded extensions of $AF$\footnote{Note that an argumentation framework always has exactly one grounded extension.}.
\end{itemize}
\end{definition}
Other semantics exist that are based on the notion of maximal conflict-freeness instead of admissibility~\cite{verheij1996two}.
\begin{definition}[Naive Set-based Argumentation Semantics~\cite{verheij1996two}]
Given an argumentation framework $AF = (AR, AT)$, a set $S \subseteq AR$ is a:
\begin{itemize}
    \item \emph{naive extension} iff $S$ is maximal w.r.t. set inclusion among all conflict-free sets. $\sigma_{na}(AF)$ denotes all naive extensions of $AF$.
\item \emph{stage extension}, iff $S \cup S^{+}$ is maximal w.r.t. set inclusion, i.e., there exists no conflict-free set $S'$ s.t. $S \cup S^+ \subset S' \cup (S')^+$. $\sigma_{st}(AF)$ denotes all stage extensions of $AF$.
\end{itemize}
\end{definition}
Let us revisit the abstract argumentation framework in the introduction to give an intuition of the behavior of different semantics.
\begin{example}
\label{ex:semantics}
    Consider the argumentation framework in Figure~\ref{fig:example-1}, which we can denote as follows:
    $$AF = (\{\arga, \argb, \argc,  \argd, \arge\},  \{(\argb, \arge), (\argc, \arge), (\argd, \arga), (\argd, \argd), (\arge, \argb), (\arge, \argc), (\arge, \arge)\})$$.
    Now, let us determine the complete, preferred, grounded, and stage extensions:
    \begin{itemize}
        \item The complete extensions are $\{\}$ and $\{\argb, \argc\}$: both sets are admissible and all arguments acceptable w.r.t. either set are actually in the corresponding set.
        \item Because $\{\}$, $\{\argb\}$, $\{\argc\}$, and $\{\argb, \argc\}$ are admissible sets and $\{\argb, \argc\}$ is the (only) $\subseteq$-maximal of these sets, it is the only preferred extension.
        \item The grounded extension is the $\subset$-minimal complete extension, i.e. $\{\}$.
        \item Finally, the only maximal conflict-free set is $\{\arga, \argb, \argc\}$, which is, thus, the only stage extension.
    \end{itemize}
\end{example}
In this paper, we examine argumentation dialogues, in which argumentation frameworks are \emph{normally expanded}, i.e., in which arguments are added to an argumentation framework, but no arguments are removed, and no attacks between existing arguments are changed.
\begin{definition}[Argumentation Framework Expansions and Normal Expansions~\cite{baumann2010expanding}]
An argumentation framework $AF' = (AR', AT')$ is:
    \begin{itemize}
        \item  an \emph{expansion} of another argumentation framework $AF = (AR, AT)$ (denoted by $AF \preceq_E AF'$) iff $AF \neq AF'$, $AR \subseteq AR'$ and $AT \subseteq AT'$;
        \item a \emph{normal expansion} of an argumentation framework $AF = (AR, AT)$ (denoted by $AF \preceq_N AF'$) iff $AF \preceq_E AF'$ and $\nexists (\arga, \argb) \in AT' \setminus AT$, such that $a \in AR \land b \in AR$.
    \end{itemize}
\end{definition}
Intuitively, an expansion adds additional arguments or attacks to an argumentation framework and a normal expansion is an expansion which does not add any attacks between two arguments that both have already been present in the initial argumentation framework.
Let us introduce an example to illustrate what expansions and normal expansions are.
\begin{example}
\label{ex:expansions}
Consider the argumentation frameworks in Figure~\ref{fig:example-exp}:
\begin{itemize}
    \item $AF_0 = (\{\arga, \argb, \argc\}, \{(\arga, \argb), (\argb, \argc)\})$;
    \item $AF_1 = (\{\arga, \argb, \argc,  \argd\}, \{(\arga, \argb), (\argb, \arga), (\argb, \argc), (\argd, \arga)\})$;
    \item $AF_2 = (\{\arga, \argb, \argc,  \argd\}, \{(\arga, \argb), (\argb, \argc), (\argd, \arga)\})$.
\end{itemize}
Clearly, $AF_1$ is an expansion of $AF_0$ ($AF_0 \preceq_E AF_1$): it has all arguments and attacks of $AF_0$, but has the additional argument $\argd$ and the additional attacks $(\argd, \arga)$ and $(\argb, \arga)$. Because of the latter attack, $AF_1$ is \emph{not} a normal expansion of $AF_0$: both $\argb$ and $\arga$ already exist in the set of arguments of $AF_0$.

In contrast, $AF_2$ \emph{is} a normal expansion of $AF_0$ ($AF_0 \preceq_N AF_2$): it has all arguments and attacks of $AF_0$ and the extra attack $(\argd, \arga)$ it adds involves at least one argument that has been added with the expansion, in this case $\argd$.
\begin{figure}[!ht]
    \subfloat[$AF_0
    $.\label{subfig:fstar}]{
        \begin{tikzpicture}[scale=0.7,
            noanode/.style={scale=0.7,dashed, circle, draw=black!60, minimum size=10mm, font=\bfseries},
            unanode/.style={scale=0.7,circle, draw=black!75, minimum size=10mm, font=\bfseries},
            invnode/.style={scale=0.7,circle, draw=white!0, minimum size=0mm, font=\bfseries},
            anode/.style={scale=0.7,,circle, fill=lightgray, draw=black!60, minimum size=10mm, font=\bfseries},
            ]
            \node[unanode]    (a)    at(2,2)  {\arga};
            \node[unanode]      (b)    at(0,2)  {\argb};
            \node[unanode]      (c)    at(0,4)  {\argc};
            \path [->, line width=0.5mm]  (a) edge node[left] {} (b);
            \path [->, line width=0.5mm]  (b) edge node[left] {} (c);
        \end{tikzpicture}
    }
    \hspace{10pt}
    \centering
    \subfloat[$AF_1
    $.\label{subfig:fstarc}]{
        \begin{tikzpicture}[scale=0.7,
            noanode/.style={scale=0.7,dashed, circle, draw=black!60, minimum size=10mm, font=\bfseries},
            unanode/.style={scale=0.7,circle, draw=black!75, minimum size=10mm, font=\bfseries},
            invnode/.style={scale=0.7,circle, draw=white!0, minimum size=0mm, font=\bfseries},
            anode/.style={scale=0.7,,circle, fill=lightgray, draw=black!60, minimum size=10mm, font=\bfseries},
            ]
            \node[unanode]    (a)    at(2,2)  {\arga};
            \node[unanode]      (b)    at(0,2)  {\argb};
            \node[unanode]      (c)    at(0,4)  {\argc};
            \node[unanode]      (d)    at(2,4)  {\argd};
            \path [->, line width=0.5mm]  (a) edge node[left] {} (b);
            \path [->, line width=0.5mm]  (b) edge node[left] {} (a);
            \path [->, line width=0.5mm]  (b) edge node[left] {} (c);
            \path [->, line width=0.5mm]  (d) edge node[left] {} (a);
        \end{tikzpicture}
    }
    \hspace{10pt}
    \centering
    \subfloat[$AF_2
    $.\label{subfig:fstard}]{
        \begin{tikzpicture}[scale=0.7,
            noanode/.style={scale=0.7,dashed, circle, draw=black!60, minimum size=10mm, font=\bfseries},
            unanode/.style={scale=0.7,circle, draw=black!75, minimum size=10mm, font=\bfseries},
            invnode/.style={scale=0.7,circle, draw=white!0, minimum size=0mm, font=\bfseries},
            anode/.style={scale=0.7,,circle, fill=lightgray, draw=black!60, minimum size=10mm, font=\bfseries},
            ]
            \node[unanode]    (a)    at(2,2)  {\arga};
            \node[unanode]      (b)    at(0,2)  {\argb};
            \node[unanode]      (c)    at(0,4)  {\argc};
            \node[unanode]      (d)    at(2,4)  {\argd};
            \path [->, line width=0.5mm]  (a) edge node[left] {} (b);
            \path [->, line width=0.5mm]  (b) edge node[left] {} (c);
            \path [->, line width=0.5mm]  (d) edge node[left] {} (a);
        \end{tikzpicture}
    }
\caption{Argumentation frameworks and their expansions: $AF_1$ is an expansion, but not a normal expansion, of $AF_0$; $AF_2$ is a normal expansion of $AF_1$.}
\label{fig:example-exp}
\end{figure}
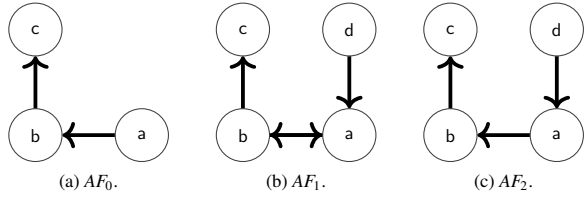
\end{example}
To formally analyze argumentation semantics, a multitude of \emph{argumentation principles} have been defined~\cite{van2017principle}.
In the context of this paper, the weak cautious monotony principle~\cite{10.1093/logcom/exab003} is of particular relevance. %
\begin{definition}[Weak Cautious Monotony~\cite{10.1093/logcom/exab003}\label{def:weak-cm}]
An argumentation semantics $\sigma$ is \emph{weakly cautiously monotonic} iff for every two argumentation frameworks $AF = (AR, AT)$, $AF' = (AR', AT')$, such that $AF \preceq_N AF'$, it holds true that $\forall E \in \sigma(AF)$, if $\{(\arga, \argb) \mid (\arga, \argb) \in AT', \arga \in AR' \setminus AR, \argb \in E \}= \{\}$ then $\exists E' \in \sigma(AF')$ such that $E \subseteq E'$.
\end{definition}
Intuitively, weak cautious monotony expects that in a normal expansion process, we can remove arguments from a previously inferred extension only if new attacks to this extension have been added.
For example, stage semantics violates weak cautious monotony, which we demonstrate using the example below.
\begin{example}
Consider the following argumentation frameworks (Figure~\ref{fig:example-cm}):
\begin{itemize}
    \item $AF' = (\{\arga, \argb\}, \{(\arga, \argb), (\argb, \arga)\})$, with the two stage extensions $\{\arga\}$ and $\{\argb\}$;
    \item $AF'' = (\{\arga, \argb, \argc\}, \{(\arga, \argb), (\argb, \arga), (\argb, \argc), (\argc, \argc)\})$ (note that $AF' \preceq_N AF''$), where the only stage extension is $\{\argb\}$.
\end{itemize}
Our agent may, when applying stage semantics, first decide to infer $\{\arga\}$ from $AF'$: given $\{\arga\}$ and $\{\argb\}$ are equally valid extensions ($\sigma_{st}(AF') = \{\{\arga\}, \{\argb\}\}$), it may make sense to force a decision by random selection for practical purposes, which may cause the agent to select $\{\arga\}$.
However, after normally expanding $AF'$ to $AF''$, the addition of the self-attacking argument $\argc$ is, due to the additional attack from $\argb$ to $\argc$, sufficient to no longer allow for the inference of any argument in the previous extension $\{\arga\}$: $\sigma_{st}(AF'') = \{\{\argb\}\}$. Weak cautious monotony requires a more compelling reason, so to speak: at least a new attacker to any argument in $\{\arga\}$ must have been added, so that such a change of inference is principle-compliant.

Note that this issue does not occur with preferred semantics: $\sigma_{pr}(AF') = \sigma_{pr}(AF'') = \{\{\arga\}, \{\argb\}\}$.
\label{ex:cautious-monotony}
\begin{figure}[!ht]
    \subfloat[$AF'
    $.\label{subfig:fprime}]{
        \begin{tikzpicture}[scale=0.7,
            noanode/.style={scale=0.7,dashed, circle, draw=black!60, minimum size=10mm, font=\bfseries},
            unanode/.style={scale=0.7,circle, draw=black!75, minimum size=10mm, font=\bfseries},
            invnode/.style={scale=0.7,circle, draw=white!0, minimum size=0mm, font=\bfseries},
            anode/.style={scale=0.7,,circle, fill=lightgray, draw=black!60, minimum size=10mm, font=\bfseries},
            ]
            \node[unanode]    (a)    at(0,2)  {\arga};
            \node[unanode]      (b)    at(2,2)  {\argb};
            \node[invnode]      (d)    at(2,-1)  {};
            \path [->, line width=0.5mm]  (a) edge node[left] {} (b);
            \path [->, line width=0.5mm]  (b) edge node[left] {} (a);
        \end{tikzpicture}
    }
    \hspace{10pt}
    \centering
    \subfloat[$AF''
    $.\label{subfig:fpprime}]{
        \begin{tikzpicture}[scale=0.7,
            noanode/.style={scale=0.7,dashed, circle, draw=black!60, minimum size=10mm, font=\bfseries},
            unanode/.style={scale=0.7,circle, draw=black!75, minimum size=10mm, font=\bfseries},
            invnode/.style={scale=0.7,circle, draw=white!0, minimum size=0mm, font=\bfseries},
            anode/.style={scale=0.7,,circle, fill=lightgray, draw=black!60, minimum size=10mm, font=\bfseries},
            ]
            \node[noanode]    (a)    at(0,2)  {\arga};
            \node[anode]      (b)    at(2,2)  {\argb};
            \node[noanode]      (c)    at(2,0)  {\argc};
            \node[invnode]      (d)    at(2,-1)  {};
            \path [->, line width=0.5mm]  (a) edge node[left] {} (b);
            \path [->, line width=0.5mm]  (b) edge node[left] {} (a);
            \path [->, line width=0.5mm]  (b) edge node[left] {} (c);
            \draw[->, line width=0.5mm] (c.-90) arc (180:180+264:4mm);
        \end{tikzpicture}
    }
\caption{Stage semantics violates cautious monotony: expanding $AF'$ with the ``self-attacking'' argument $\argc$ and an attack from $\argb$ to $\argc$ is sufficient so that we can no longer infer any argument of our initial extension $\{\arga\}$. An addition of a direct attack on the extension is not required.}
\label{fig:example-cm}
\end{figure}
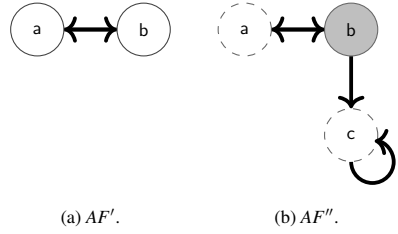
\end{example}
An extension of abstract argumentation is \emph{value-based argumentation}~\cite{bench2003persuasion}. A value-based argumentation-framework is a five-tuple $VAF = (AR, AT, V, val, {\cal P })$, where $AR$ is a finite set of arguments, $AT \subseteq (AR \times AR)$, such that $\forall \arga \in AR, (\arga, \arga) \not \in AT$ (i.e., $AT$ is irreflexive: there are no ``self-attacking'' arguments), $V$ is a non-empty set of \emph{values}, $val$ is a function that maps elements of $AR$ to elements of $V$ (i.e., given any $\arga \in AR$, $val(\arga)$ returns a value in $V$), and $ {\cal P }$ is a sequence (totally ordered multiset) of binary relations on $V$ (called \emph{value preferences}), ${\cal P } = \langle P_0, ..., P_n \rangle$, such that every binary relation $P_i, 0 \leq i \leq n$ is transitive, irreflexive, and asymmetric (\emph{i)} $\forall x,y,z \in V$, if $(x, y) \in P_i  \land (y, z) \in P_i$ then $(x, z) \in P_i$; \emph{ii)} $(x, x) \not \in P_i$; \emph{iii)}  if $(x, y) \in P_i$ then $(y, x) \not \in P_i$); $AF_{P_i}$ is the  argumentation framework $(AR, AT')$, such that $AT' = \{(\arga, \argb) | (\arga, \argb) \in AT,  (val(\argb), val(\arga)) \not \in P_i \}$. We call $AF_{p_i}$ the \emph{subjective} argumentation framework of $P_i$ w.r.t. $VAF$.
For an example of a value-based argumentation framework, we refer the reader back to Example~\ref{example-3}.

%
\section{Degrees of Agreements in Abstract  Argumentation}
\label{sec:argumentation}
%
Let us formalize the intuitions introduced in Section~\ref{intro}, starting with the notion of an argumentation-based agreement scenario.
\begin{definition}[Argumentation-based Agreement Scenario (AAS)]
    An argumentation-based agreement scenario is a tuple $AAS = (AF, T, SIG)$, where $AF = (AR, AT)$ is an argumentation framework, $T \subseteq AR$ and $SIG = \langle \sigma_0, ... , \sigma_n \rangle$, such that each $\sigma_i, 0 \leq i \leq n$ is an argumentation semantics. We say that $T$ denotes the \emph{topic} of $AAS$.
\end{definition}
Intuitively, an agreement scenario contains the argumentation framework that (typically several) agents, each represented by the argumentation semantics, infer extensions from; the topic is the subset of arguments in the argumentation that the agents want to agree on. Arguments that are not in the topic play, in contrast, a merely auxiliary role in that they may have an impact on which topic arguments are inferred. Let us recall the motivating example.
\begin{example}
    In Example~\ref{example:intro} (see Figure~\ref{fig:example-1} for the argumentation framework), we have the agreement scenario $((\{\arga, \argb, \argc,  \argd, \arge\},  \{(\argb, \arge), (\argc, \arge), (\argd, \arga), (\argd, \argd), (\arge, \argb), (\arge, \argc), (\arge, \arge)\}), \{\arga, \argb, \argc\}, \\ \langle \sigma_{st}, \sigma_{pr}, \sigma_{gr} \rangle)$: our three agents apply stable, preferred, and grounded semantics, respectively, and are interested in the topic arguments $\arga$, $\argb$, and $\argc$.
\end{example}
We now introduce a measure of the degree of satisfaction given an argumentation semantics w.r.t. an argumentation framework, and two subsets of the argumentation framework's arguments (a topic set and a conclusion set, the latter of which is typically an extension inferred by applying an argumentation semantics).
First, we introduce the degree of satisfaction as an abstract measure.
\begin{definition}[Abstract Degree of Satisfaction]\label{adegree-sat}
Let $\sigma$ be an argumentation semantics, let $AF = (AR, AT)$ be an argumentation framework, and let $T, S \subseteq AR$. The degree of satisfaction of $\sigma$ w.r.t. $AF$, $T$, and $S$, denoted by $\phi^{sim}_{\sigma}(AF, T, S)$, is determined as follows:
\begin{align*}
  \phi_{\sigma}^{sim}(AF, T, S) =
  max(\{sim(E,S,T) | E \in \sigma(AF)\}),
\end{align*}
where $sim: 2^{\cal AR} \times 2^{\cal AR} \times 2^{\cal AR} \rightarrow [0, 1]$. We call $sim$ the \emph{similarity function} of $\phi_{\sigma}^{sim}$.
\end{definition}
We then define several specific measures of the degree of satisfaction, using a measure that is based on the well-known \emph{Hamming distance} measure~\cite{6772729} (which is also used in related research on \emph{merging argumentation frameworks}~\cite{10.5555/3032027.3032032}), as well as simpler intersection-based measures.
\begin{definition}[Degree of Satisfaction, Similarity Measures]\label{degree-sat}
Given the sets of arguments $T, E, S \in 2^{\cal AR}$, as well as the set differences (relative complements) $E' := T \setminus E$ and $S' := T \setminus S$, we define the following similarity functions of \emph{extensions} $E$ and $S$ with respect to the \emph{topic} $T$.
\begin{description}
    \item[Intersection-based similarity] ($i$-similarity):
    \begin{align*}
        i(E,S,T) =
        \begin{cases}
            1 & \text{if  } |T \cap (E \cup S)| = 0; \\
            \frac{|T \cap E \cap S|}{|T \cap (E \cup S)|} & \text{otherwise}.
        \end{cases}    
    \end{align*}
    \item[Complement-based similarity] ($c$-similarity):
    \begin{align*}
        c(E,S,T) =
        \begin{cases}
            1 & \text{if  } |T \cap (E' \cup S')| = 0; \\
            \frac{|T \cap E' \cap S'|}{|T \cap (E' \cup S')|} & \text{otherwise}.
        \end{cases}
    \end{align*}
    \item[Hamming-based similarity] ($h$-similarity):
    \begin{align*}
        h(E,S,T) =
        \begin{cases}
            1 & \text{if  } |T| = 0; \\
            \frac{|T \cap E \cap S| + |(E') \cap (S')|}{|T|} & \text{otherwise}.
        \end{cases}    
    \end{align*}
\end{description}
\end{definition}
Here, the intersection-based ($i$-similarity) and complement-based similarity ($c$-similarity) serve as mere building blocks for the Hamming similarity-like measure; we merely use both $i$-similarity and $h$-similarity in Example~\ref{example-aas-1} to highlight the nuanced, yet important, difference.
Let us note that the measures could potentially be extended to account for weights that model the importance of arguments; intuitively, the current measures consider all arguments in the topic as equally important, whereas all arguments that are not in the topic are considered as entirely irrelevant.

Now, we can introduce the two-agent degree of satisfaction for AAS, as a means to provide a starting point for the $n$-agent degrees that follow later.
\begin{definition}[Two-Agent Degree of Satisfaction]
\label{def:two-agent-degree}
Let $AF = (AR, AT)$ be an argumentation framework, let $T \subseteq AR$, let $\sigma$ and $\sigma'$ be argumentation semantics, and let $sim$ be a similarity function.
The degree of satisfaction between $\sigma$ and $\sigma'$ w.r.t. $AF$ and $T$, denoted by $sat^{sim}(AF, T, \sigma, \sigma')$, is determined as follows:
\begin{align*}
    sat^{sim}(AF, T, \sigma, \sigma') = max(\{\phi^{sim}_{\sigma}(AF, T, E)| E \in \sigma'(AF) \}).
\end{align*}
\end{definition}
Intuitively, the two-agent degree of satisfaction measures the distance between the maximally similar extensions (considering only the topic arguments of the extensions) that the agents' semantics can infer from the given argumentation framework.
Let us introduce an example for illustration purposes.
\begin{example}\label{example-aas-1}
Consider again the argumentation framework $AF = (\{\arga, \argb, \argc,  \argd, \arge\},  \{(\argb, \arge), (\argc, \arge), (\argd, \arga), (\argd, \argd), (\arge, \argb), (\arge, \argc), (\arge, \arge)\})$, as depicted in Figure~\ref{fig:example-1}, and let $T = \{\arga, \argb, \argc\}$. Let $AAS = (AF, T, \langle \sigma_{st}, \sigma_{pr}, \sigma_{gr}\rangle)$.
Recall that the extensions that the semantics infer are the following:
\begin{itemize}
    \item $\sigma_{st} = \{\{\arga, \argb, \argc\}\}$; 
    \item $\sigma_{pr} = \{\{\argb, \argc\}\}$; 
    \item $\sigma_{gr} = \{\{\}\}$.
\end{itemize}
Table~\ref{table:3} shows all two-agent degrees of satisfaction of the $AAS$, given the $h$-similarity and $i$-similarity functions (the degree of the latter in parentheses).
We can observe that the latter two function arguments (the two argumentation semantics) of the two-agent degree of satisfaction are commutative: for every $AF, AF' \in {\cal F}$ and every $\sigma, \sigma' \in \cal S$ it holds that $sat^{sim}(AF, T, \sigma, \sigma') = sat^{sim}(AF, T, \sigma', \sigma)$.
\begin{table*}
\caption{Matrix of degrees of satisfaction between the three agents in the argumentation-based agreement scenario from Example~\ref{example-aas-1}, given the h-similarity and the i-similarity functions (the degree of the latter in parentheses).}
\label{table:3}
\centering
\renewcommand{\arraystretch}{1.0}
\begin{tabular}{ |c|c|c|c|} 
 \hline
 \phantom{ } & $\sigma_{st}$ & $\sigma_{pr}$  & $\sigma_{gr}$\\ 
  \hline
 $\sigma_{st}$ & $1$ ($1$) & $\frac{2}{3}$ ($\frac{2}{3}$) & $0$ ($0$) \\ 
 $\sigma_{pr}$ & $\frac{2}{3}$ ($\frac{2}{3}$) & $1$ ($1$) & $\frac{1}{3}$ ($0$) \\
 $\sigma_{gr}$ & $0$ ($0$) & $\frac{1}{3}$ ($0$) & $1$ ($1$) \\ 
 \hline
\end{tabular}
\end{table*}
\end{example}
The situation is more complicated when we move from two to $n \in \mathbb{N}$, $n > 2$ agents: then, we cannot simply bi-laterally maximize similarity, but instead need to maximize an aggregated measure.
We define degrees of minimal, mean, and median agreement for argumentation-based agreement scenarios. Note that in the definition we make use of the median of a sequence (ordered multiset) of real numbers $K$, denoted by $med(K)$.
\begin{definition}[Degrees of Minimal, Mean, and Median Agreement]\label{degs-aas}
Let $AAS = (AF, T, SIG)$ be an argumentation-based agreement scenario, with $SIG = \langle \sigma_0, ..., \sigma_n \rangle$, and let $sim$ be a similarity function (cf. Definition~\ref{adegree-sat}).
The degree of minimal agreement of $AAS$, denoted by $deg^{sim}_{min}(AAS)$, is determined as follows:
    \begin{align*}
        deg^{sim}_{min}(AAS) =  max(\{ min(\{\phi^{sim}_{\sigma_i}(AF, T, E) | 0 \leq i \leq n \}) | E \in 2^{T} \}). 
    \end{align*}
The degree of mean agreement of $AAS$, denoted by $deg_{mean}(AAS)$, is determined as follows:
    \begin{align*}
        deg^{sim}_{mean}(AAS) = max(\{ \phi'_{avg}(E)| E \in 2^{T} \}), 
    \end{align*}
where $\phi'_{avg}(E) = \frac{\sum_{i = 0}^{n} \phi^{sim}_{\sigma_i}(AF, T, E)}{n+1}$.

The degree of median agreement of $AAS$, denoted by $deg_{med}(AAS)$, is determined as follows:
    \begin{align*}
        deg^{sim}_{med}(AAS) = max(med(\langle
        s_0(E), ..., s_n(E)
        \rangle) | E \in 2^{T}),
    \end{align*}
where for each $s_j(E), 0 \leq j \leq n$, $s_j(E) = \phi^{sim}_{\sigma_j}(AF, T, E)$ and for each $s_k(E), 0 < k \leq n$, $s_k(E) \geq s_{k-1}(E)$.
\end{definition}
Let us claim that intuitively, the degree of minimal agreement is the most sensitive to outliers: if most (of many) agents agree on what arguments to infer from the topic but one agent disagrees entirely, the  degree of minimal agreement is affected substantially.
In contrast, the degree of median agreement is the least sensitive, i.e., it may not be affected at all in such cases.
\begin{example}
\label{ex:degrees-rev}
Considering the AAS in Example~\ref{example-aas-1} and the $h$-similarity (as well as the $i$-similarity) function, we have the following degrees of agreement:
\begin{itemize}
    \item degree of minimal agreement: $\frac{1}{3}$;
    \item degrees of mean and median agreement: $\frac{2}{3}$.
\end{itemize}
\end{example}
%

\section{Expanding Argumentation-based Agreement Scenarios}
\label{exp-aas}
In our study, we are interested in how degrees of agreement change in dynamic environments, in which agents have a dialogue by exchanging arguments.
In our case, the dynamism is modeled by a normally expanding argumentation framework, i.e., new arguments are added to an initial argumentation framework, as well as attacks involving these new arguments as source or targets, whereas the attack relationships between initially existing arguments remain unchanged.
In order to model how agreement scenarios \emph{change} in the face of new information that becomes available, let us define normal expansions in the context of AAS.
\begin{definition}[AAS Normal Expansions]
Let $AAS = (AF, T, SIG)$ and $AAS' = (AF', T', SIG')$ be argumentation-based agreement scenarios, with $AF = (AR, AT)$.
$AAS'$ is a \emph{normal expansion} of $AAS$, denoted by $AAS \preceq_N AAS'$, iff it holds true that $AF \preceq_N AF'$, $T \subseteq T', (T' \setminus T) \cap AR = \emptyset$ and $SIG = SIG'$.
\end{definition}
The idea behind an AAS normal expansion is that agents engage in an argumentation-based \emph{dialogue} that starts with a topic set (a subset of the arguments of the AAS' argumentation framework). The agents argue by adding new arguments to the argumentation framework, and may also expand the topic set to include some of these new arguments; in contrast, the topic cannot be expanded using arguments from the initial argumentation framework. Also, the agents cannot ``remove'' arguments (neither from the argumentation framework nor from the topic set), nor ``change'' the attack relations between existing arguments or add existing arguments to the topic set.

Let us illustrate the notion of an AAS normal expansion by introducing an example.
\begin{example}\label{ex:aas-expansion}
Let us revisit the argumentation frameworks depicted in Figure~\ref{fig:example-exp}:
\begin{itemize}
    \item $AF_0 = (\{\arga, \argb, \argc\}, \{(\arga, \argb), (\argb, \argc)\})$;
    \item $AF_1 = (\{\arga, \argb, \argc,  \argd\}, \{(\arga, \argb), (\argb, \arga), (\argb, \argc), (\argd, \arga)\})$;
    \item $AF_2 = (\{\arga, \argb, \argc,  \argd\}, \{(\arga, \argb), (\argb, \argc), (\argd, \arga)\})$.
\end{itemize}
Here, $AF_0$ is the argumentation framework in our initial agreement scenario, whereas $AF_1$ and $AF_2$ are then used in potential expansions.
We utilize these argumentation frameworks in the following AAS:
\begin{itemize}
    \item $AAS_0 = (AF_0, \{\arga, \argb\}, \langle \sigma_0, \sigma_1 \rangle)$;
    \item $AAS_1 = (AF_1, \{\arga, \argb\}, \langle \sigma_0, \sigma_1 \rangle)$;
    \item $AAS_2 = (AF_2, \{\arga, \argb, \argd\}, \langle \sigma_0, \sigma_1 \rangle)$;
    \item $AAS_2' = (AF_2, \{\arga, \argb, \argc\}, \langle \sigma_0, \sigma_1 \rangle)$;
    \item $AAS_2'' = (AF_2, \{\arga, \argb\}, \langle \sigma_0, \sigma_1, \sigma_2 \rangle)$.
\end{itemize}
Here, we do not care about the specifics of the argumentation semantics $\sigma_0$, $\sigma_1$, and $\sigma_2$.
Now, we can observe the following:
\begin{itemize}
    \item $AAS_0 \preceq_N AAS_1$ does not hold: the expansion from $AF_0$ to $AF_1$ ``adds'' an attack between arguments in the initial argumentation framework;
    \item $AAS_0 \preceq_N AAS_2$ holds: $AF_0 \preceq_N AF_1$ holds, the topic is expanded by adding a ``new'' argument, and the sequence of semantics remains the same; 
    \item $AAS_0 \preceq_N AAS_2'$ does not hold: the topic is expanded by adding an argument from the initial argumentation framework;
    \item $AAS_0 \preceq_N AAS_2''$ does not hold: the sequence of semantics is changed by appending a third semantics (representing a third agent).
\end{itemize}
\end{example}
Let us show that formal argumentation principles can help guarantee, given an argumentation-based agreement scenario $AAS$ and any of its normal expansions $AAS'$, that there is an upper bound to the difference between the degrees of minimal 
agreements of $AAS$ and $AAS'$, given some constraints; this observation helps us understand that in many (dynamic) scenarios, agreements remain somewhat stable unless there is a substantial compelling event that leads to the violation of the aforementioned constraints.

We first introduce a class of argumentation principles, which we call \emph{relaxed monotony principles}.
\begin{definition}[Relaxed Monotony Principle]\label{relaxed-monotony}
An argumentation semantics $\sigma$ satisfies the relaxed monotony principle $RMP_{p}$ iff for every two argumentation frameworks $AF = (AR, AT)$, $AF' = (AR', AT')$ such that $AF \preceq_N AF'$, the following statement holds true:
\begin{align*}
    &{} \forall E \in \sigma(AF), \\
    &{} {if } \: p(AF, AF', E, \sigma) \text{ holds true} \\ &{} 
    \text{ then } \exists E' \in \sigma(AF') \text{ s.t.\ } E \subseteq E',
\end{align*}
where $p$ (the principle's $p$-function) is a Boolean function $p: {\cal F} \times {\cal F} \times 2^{\cal AR} \times {\cal S} \rightarrow \{true, false\}$, i.e., $AF, AF' \in {\cal F}$, $E \in 2^{\cal AR}$ (also: $E \subseteq AR$), and $\sigma \in \cal S$.
\end{definition}
We may call the $p$-function of a relaxed monotony principle the principle's \emph{monotony condition}; given an argumentation semantics $\sigma$ that satisfies a relaxed monotony principle $RMP_{p}$, we may say that $\sigma$ satisfies $p$-relaxed monotony.
Now, we can characterize weak cautions monotony (Definition~\ref{def:weak-cm}) as a relaxed monotony principle.
\begin{restatable}{reprop}{wcmprop}\label{prop-wcm}
An argumentation semantics $\sigma$ satisfies weak cautious monotony iff $\sigma$ satisfies the relaxed monotony principle $RMP_{cm}$, where the $p$-function is characterized by the following function:
    \begin{align*}
        &{} cm(AF, AF', E, \sigma) = \begin{cases}
      true, \text{if  } S_{*} = \emptyset;  &\\
      false,\text{otherwise},
    \end{cases}
    \end{align*}
    $AF = (AR, AT), AF' = (AR', AT')$, and $S_{*} = \{(\arga, \argb) | (\arga, \argb) \in AT', a \in AR' \setminus AR, b \in E \}$.
\end{restatable}
All proofs are provided in the appendix.

For the sake of increasing the stability of the agreements that agents reach, we may want to introduce stricter relaxed monotony principles to then enforce that agents adhere to them in the context of our agreement scenarios.
Let us, for this purpose, introduce a principle that (roughly speaking) stipulates that we can only relax monotony if the expansion of our initial argumentation framework implies that an argument that attacks our initially inferred extension cannot be rejected.
For this, let us first introduce the notion of a \emph{strong attacker} in the context of argumentation framework expansions.
\begin{definition}[Strong Attacker]
Let $AF = (AR, AT)$ and $AF' = (AR', AT')$ be argumentation frameworks, such that $AF \preceq_E AF'$ and let $S \subseteq AR$. $\arga \in AR'$ is a strong attacker of $S$ w.r.t. $AF, AF'$, denoted by $\arga \rightarrow_{AF, AF'} S$, iff $\arga$ attacks $S$ and $\exists S' \subseteq AR'$, such that $S'$ is conflict-free in $AF'$ and $S'$ strongly defends $\arga$.
\end{definition}
The principle's definition follows below.
\begin{definition}[Strong Relaxed Monotony]\label{causal-ref-dep} An argumentation semantics $\sigma$ satisfies strong relaxed monotony iff $\sigma$ satisfies the relaxed monotony principle $RMP_{srm}$ s.t. the $p$-function is characterized as follows:
    \begin{align*}
       &{} srm(AF, AF', E, \sigma) = \begin{cases}
      true & \text{if  } \nexists E' \in \sigma(AF') \text{ s.t.\ } \exists \arga \in E', \arga \rightarrow_{AF, AF'} E; \\
      false & \text{otherwise}.
    \end{cases}
    \end{align*}
\end{definition}
Let us note that strong relaxed monotony is not necessarily stricter than weak cautious monotony: intuitively, the new strong attacker of an extension whose presence implies that we may reject arguments that we have previously inferred does not necessarily stem from a new direct attack on the initially inferred extension.
For example, we may infer the stage extension $\{\arga\}$ from the argumentation framework $AF = (\{\arga, \argb, \argc\}, \{(\arga, \argb) , (\argb, \argc), (\argc, \arga)\})$ but when expanding $AF$ to $AF' = (\{\arga, \argb, \argc, \argd\},  \{(\arga, \argb) , (\argb, \argc), (\argc, \arga)\}, \{\argd, \argb\})$, the only stage extension is $\{\argc, \argd\}$. The $p$-function of weak cautious monotony evaluates to true, because no argument that attacks $\arga$ has been added with the normal expansion; in contrast, the $p$-function of strong relaxed monotony evaluates to false because $\argc$ is in the stage extension of $AF'$ and a strong attacker of $\arga$.

The only argumentation semantics (of the ones whose definitions we provide) that satisfies strong relaxed monotony is naive semantics, which indeed satisfies every relaxed monotony principle.
\begin{restatable}{reprop}{rmpprop}\label{prop-rmp}
For every $AF = (AR, AT), AF' = (AR', AT')$, such that $AF \preceq_N AF'$, the following statement holds true for every relaxed monotony principle $RMP_p$:
\begin{align*}
    &{} \forall E \in \sigma_{na}(AF), \\
    &{} {if } \: p(AF, AF', E, \sigma_{na}) \text{ holds true} \\ &{} 
    \text{ then } \exists E' \in \sigma_{na}(AF') \text{ s.t.\ } E \subseteq E'.
\end{align*}
\end{restatable}
By example, we can show that complete, preferred, grounded, and stage semantics do not satisfy strong relaxed monotony.
\begin{example}\label{ex:srm}
Consider the argumentation framework $AF = (AR, AT) = (\{\arga\}, \{\})$, $AF' = (AR', AT') = (\{\arga, \argb, \argc\}$, $\{(\argb, \arga), (\argb, \argc), (c, a), (\argc, \argb)\})$ (Figure~\ref{fig:example-srm}). Note that $AF \preceq_N AF'$.
We observe the following, in absence of a new strong attacker of $\{\arga\}$ in the (normal) expansion of $AF$ to $AF'$:
\begin{itemize}
    \item $\sigma_{co}(AF) = \{\{\arga\}\}$ and $\sigma_{co}(AF') =  \{\{\argb\}, \{\argc\}, \{\}\}$;
    \item $\sigma_{pr}(AF) = \{\{\arga\}\}$ and $\sigma_{pr}(AF') = \{\{\argb\}, \{\argc\}\}$;
    \item $\sigma_{gr}(AF) = \{\{\arga\}\}$ and $\sigma_{gr}(AF') = \{\}$;
    \item $\sigma_{st}(AF) = \{\{\arga\}\}$ and $\sigma_{st}(AF') = \{\{\argb\}, \{\argc\}\}$.
\end{itemize}
\begin{figure}[!ht]

\subfloat[$AF
    $.\label{subfig:af}]{
        \begin{tikzpicture}[scale=0.7,
            noanode/.style={scale=0.7,dashed, circle, draw=black!60, minimum size=10mm, font=\bfseries},
            unanode/.style={scale=0.7,circle, draw=black!75, minimum size=10mm, font=\bfseries},
            invnode/.style={scale=0.7,circle, draw=white!0, minimum size=0mm, font=\bfseries},
            anode/.style={scale=0.7,,circle, fill=lightgray, draw=black!60, minimum size=10mm, font=\bfseries},
            ]
            \node[anode]    (a)    at(2,2)  {\arga};
        \end{tikzpicture}
    }
    \hspace{10pt}
    \centering
    \subfloat[$AF'
    $.\label{subfig:afprime}]{
        \begin{tikzpicture}[scale=0.7,
            noanode/.style={scale=0.7,dashed, circle, draw=black!60, minimum size=10mm, font=\bfseries},
            unanode/.style={scale=0.7,circle, draw=black!75, minimum size=10mm, font=\bfseries},
            invnode/.style={scale=0.7,circle, draw=white!0, minimum size=0mm, font=\bfseries},
            anode/.style={scale=0.7,,circle, fill=lightgray, draw=black!60, minimum size=10mm, font=\bfseries},
            ]
            \node[noanode]    (a)    at(0,4)  {\arga};
            \node[unanode]    (b)    at(2,4)  {\argb};
            \node[unanode]    (c)    at(2,2)  {\argc};
            \path [->, line width=0.5mm]  (b) edge node[left] {} (a);
            \path [->, line width=0.5mm]  (b) edge node[left] {} (c);
            \path [->, line width=0.5mm]  (c) edge node[left] {} (a);
            \path [->, line width=0.5mm]  (c) edge node[left] {} (b);
        \end{tikzpicture}
    }
\caption{Example~\ref{ex:srm}: violation of strong relaxed monotony, here assuming we apply complete, preferred, or stage semantics.}
\label{fig:example-srm}
\end{figure}
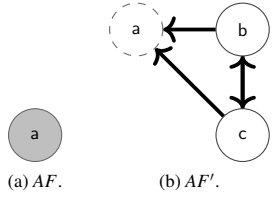
\end{example}

Let us claim that the same example shows that the strong relaxed monotony principle is not satisfied by any semantics surveyed in~\cite{baroni2018abstract} (with the exception of naive semantics), and neither by the weakly admissible set-based semantics that have been introduced in~\cite{baumann2020revisiting}.
However, even when semantics that do not satisfy this principle (or any other relaxed monotony principle) are used, the agents can still commit to complying with the principle.
We illustrate this by example.
\begin{example}
\label{ex:relaxed-monotony}
Let us again consider the introductory example. Recall that we use the $h$-similarity function (Definition~\ref{degree-sat}) as our similarity measure.
However, now we go a step back into the past, and assume that the agents started debating the following argumentation framework first: $AF = (AR, AT) = (\{\arga, \argb, \argc\}, \{\})$.
Considering the introductory example, this gives us the following argumentation-based agreement scenario: $AAS = (AF, \{\arga, \argb, \argc\}, \langle \sigma_{st}, \sigma_{pr}, \sigma_{gr} \rangle)$. Because $\sigma_{pr}(AF) = \sigma_{st}(AF) = \sigma_{gr}(AF) = \{\{\arga, \argb, \argc\}\}$, the degrees of minimal, mean, and median agreement are all $1$.
We expand $AAS$ and create a new argumentation-based agreement scenario $AAS' = (AF', \{\arga, \argb, \argc\}, \langle \sigma_{st}, \sigma_{pr}, \sigma_{gr} \rangle))$, where $AF' = (AR', AT') = (\{\arga, \argb, \argc,  \argd, \arge\},  \{(\argb, \arge), (\argc, \arge), (\argd, \arga), (\argd, \argd), (\arge, \argb), (\arge, \argc), (\arge, \arge)\})$ (see Figure~\ref{fig:example-1}).
Without any further considerations, this would have severe effects on the degrees of agreement, given the $h$-similarity function:
\begin{itemize}
    \item $1 - |deg^{h}_{min}(AAS) - deg^{h}_{min}(AAS')| = \frac{1}{3}$;
    \item $1 - |deg^{h}_{mean}(AAS) - deg^{h}_{mean}(AAS')| = \frac{2}{3}$;
    \item $1 - |deg^{h}_{med}(AAS) - deg^{h}_{med}(AAS')| = \frac{2}{3}$.
\end{itemize}
To prevent this, let us assume all agents have committed to comply with the strong relaxed monotony principle $RMP_{srm}$, i.e., they adjust their inferences for the sake of keeping the degree of agreement stable. This means that when expanding $AAS$ to $AAS'$, each agent checks if $srm'(AF, AF', ES, \sigma)$ holds true, given the agent's argumentation semantics $\sigma$ and $ES = \sigma(AF')$:
\begin{align*}
srm'(AF, AF', ES, \sigma) = 
\begin{cases}
      true & \text{if  } \forall E \in \sigma(AF), \\
       & \quad \text{ if } srm(AF, AF', E, \sigma) \text{ holds true} \\
        & \quad \text{then } \exists E' \in ES \text{ s.t.\ } E \subseteq E'; \\
      false & \text{otherwise}.
\end{cases}
\end{align*}
Note that in contrast to $srm$, $srm'$ is \emph{not} the p-function of a relaxed monotony principle.
If this is not the case, the agent adjusts its argumentation semantics $\sigma$ to $\sigma^{*}$, such that for $ES^{*} = \sigma^{*}(AF')$, the following holds true:
\begin{itemize}
    \item $srm'(AF, AF', ES^{*}, \sigma)$ holds true;
    \item There exists no argumentation semantics $\sigma^{**}$, such that $srm'(AF, AF', \sigma^{**}(AF'), \sigma)$ holds true and $\exists E^{**} \in \sigma^{**}(AF')$, such that $(E^{**} \not \in \sigma^{*}(AF')$, $\phi_{\sigma}(AF', T', E^{**}) \geq sat^h(AF', T', \sigma, \sigma^{*}))$ and $\nexists E^{*} \in \sigma^{*}(AF')$, \emph{such that} $\phi_{\sigma}(AF', T', E^{*}) \leq sat^h(AF', T', \sigma, \sigma^{*})$.
\end{itemize}
Recall that $sat^h$ denotes the two-agent degree of satisfaction (Definition~\ref{def:two-agent-degree}), given $h$ as our similarity measure. Roughly speaking, $\sigma^{*}$ satisfies the relaxed monotony condition in the specific case of $AAS$  and $AAS'$, and all extensions in $\sigma^{*}(AF')$ provide maximal satisfaction (considering that the relaxed monotony constraint must be satisfied), given the topic set $T'$ and the agent's ``original'' semantics.
In our example, the agents can adjust $\sigma_{pr}$ and $\sigma_{gr}$ to $\sigma_{pr}^{*}$ and $\sigma_{gr}^{*}$, such that $\sigma_{st}(AF') = \sigma_{pr}^{*}(AF') = \sigma_{gr}^{*}(AF') = \{\{a, b, c,\}\}$
 and hence:
 \begin{align*}
    &{} |deg^{h}_{min}(AAS) - deg^{h}_{min}(AAS')| = \\
    &{} |deg^{h}_{mean}(AAS) - deg^{h}_{mean}(AAS')| = \\ 
    &{} |deg^{h}_{med}(AAS) - deg^{h}_{med}(AAS')| = 0.
\end{align*}
\end{example}
In the example above, we are interested in how the degree of agreement changes as the result of an agreement scenario expansion.
To facilitate conciseness, let us introduce an abstraction for such change.
\begin{definition}[Agreement Delta]\label{def:aas-delta}
   Let $AAS$ and $AAS'$ be argumentation-based agreement scenarios and let $sim$ be a similarity function. We define the agreement delta of $AAS$ and $AAS'$ w.r.t. $sim$ and a degree of agreement $deg^{sim} \in \{deg^{sim}_{min}, deg^{sim}_{mean}, deg^{sim}_{med}\}$, denoted by $\Delta_{deg^{sim}}(AAS, AAS')$, as $|deg^{sim}(AAS) - deg^{sim}(AAS')|$.
\end{definition}
Let us revisit a previous example to illustrate the notion of an agreement delta.
\begin{example}
\label{ex:agreement-delta}
Consider again the following two AAS  (similar to Example~\ref{ex:relaxed-monotony}):
\begin{itemize}
    \item $AAS = (AF, \{\arga, \argb, \argc\}, \langle \sigma_{st}, \sigma_{pr}, \sigma_{gr} \rangle)$, where $AF = \{\arga, \argb, \argc\}, \{\})$;
    \item $AAS' = (AF', \{\arga, \argb, \argc\}, \langle \sigma_{st}, \sigma_{pr}, \sigma_{gr} \rangle))$, where $AF' = (AR', AT') = $ \\ $ (\{\arga, \argb, \argc,  \argd, \arge\},  \{(\argb, \arge), (\argc, \arge), (\argd, \arga), (\argd, \argd), (\arge, \argb), (\arge, \argc), (\arge, \arge)\})$.
\end{itemize}
Clearly, we have $deg^{h}_{min}(AAS) = deg^{h}_{mean}(AAS) = deg^{h}_{med}(AAS) = 1$ (because $\sigma_{st}(AF) = \sigma_{pr}(AF) = \sigma_{gr}(AF) = \{\arga, \argb, \argc\}$).
From a previous example (Example~\ref{ex:degrees-rev}), we know that $deg^{h}_{min}(AAS') = \frac{1}{3}$ and $deg^{h}_{mean}(AAS') = deg^{h}_{med}(AAS') = \frac{2}{3}$. Hence, $\Delta_{deg^{h}_{min}} = |1 - \frac{1}{3}| = \frac{2}{3}$ and $\Delta_{deg^{h}_{mean}} = \Delta_{deg^{h}_{med}} = |1 - \frac{2}{3}| = \frac{1}{3}$.
\end{example}
Every relaxed monotony principle $RMP_{p
}$ guarantees the following, given a non-violated monotony condition, two argumentation-based agreement scenarios $AAS = (AF, T, SIG)$ and $AAS' = (AF', T', SIG')$ such that $AAS \preceq_N AAS'$, and the $h$-similarity function:
\begin{itemize}
    \item If the topic does not change and all agents agree that all topic arguments should be inferred, then we can guarantee that our degrees of minimal, mean, and median agreement remain the same, i.e. their agreement delta is $0$;
    \item In turn, the guarantee we can provide for the stability of the degree of minimal agreement depends on the relative number of topic arguments that all agents agree on inferring, as well as on the stability of the topic (i.e., on whether and how many topic arguments we add to our agreement scenario when expanding).
\end{itemize}
Let us formalize the first observation.
\begin{restatable}{reprop}{propaasrelaxedonea}\label{proposition-3a}
Let $AAS =$ $(AF, T, SIG)$ and $AAS' = (AF', T', SIG')$ be argumentation-based agreement scenarios s.t.\ $AAS \preceq_N AAS'$, let $T = T'$, and let $h$ be the $h$-similarity function. Let $RMP_{p}$ be a relaxed monotony principle and let it hold true that, given $SIG = \langle \sigma_0, ..., \sigma_n \rangle$, each $\sigma_i, 0 \leq i \leq n$ satisfies $RMP_{p}$.
If for each  $\sigma_i$, $\exists E \in \sigma_i(AF)$ s.t. $T \subseteq E$ and $p(AF, AF', E, \sigma_i)$ holds true then $\deg^{h}_{min}(AAS) = \deg^{h}_{min}(AAS') = \deg^{h}_{mean}(AAS) = \deg^{h}_{mean}(AAS') = \deg^{h}_{med}(AAS) = \deg^{h}_{med}(AAS') = 1$.
\end{restatable}
Now, as a prerequisite for giving further bounds of change for the degree of \emph{minimal agreement}, let us show that its (tight) lower bound is, given an argumentation-based agreement scenario $AAS = (AF, T, SIG)$, $\frac{\lfloor |T| / 2 \rfloor}{|T|}$ (we denote the floor of a real number $r$ by $\lfloor r \rfloor$).
Note that because we divide by $|T|$ and later by $|T'|$, given a topic $T'$, we assume that $T, T' \neq \emptyset$ to avoid division by zero.
\begin{restatable}{relemma}{minmindeg}\label{lemma:min-min-deg}
Let $AAS =$ $(AF, T, SIG)$ be an argumentation-based agreement scenario and let $h$ be the $h$-similarity function.
It holds that the tight lower bound of $\deg^{h}_{min}(AAS)$ is $\frac{\lfloor |T| / 2 \rfloor}{|T|}$.
\end{restatable}
Now, we can go one step further and show that given two argumentation-based agreement scenarios $AAS = (AF, T, SIG)$ and $AAS' = (AF', T', SIG')$ such that $AAS \preceq_N AAS'$,
the tight upper bound for change to the degree of minimal agreement is, given a \emph{decrease} in the degree of minimal agreement, $1 - \frac{\lfloor |T'| / 2 \rfloor}{|T'|}$.
\begin{restatable}{relemma}{minmindegchange}\label{lemma:min-min-deg-change}
Let $AAS = (AF, T, SIG)$ and $AAS' = (AF', T', SIG')$ be argumentation-based agreement scenarios s.t. $AAS \preceq_N AAS'$ and let $h$ be the $h$-similarity function.
Let us assume that $deg^{h}_{min}(AAS) \geq deg^{h}_{min}(AAS')$.
The tight upper bound of $\Delta_{\deg^{h}_{min}}(AAS, AAS')$ is $1 - \frac{\lfloor |T'| / 2 \rfloor}{|T'|}$.
\end{restatable}
Finally, we can develop the intuition that the tight upper bound for \emph{negative} change in the degree of minimal agreement given the $h$-similarity function depends on how many arguments of the topic our agents can clearly agree on inferring: the fewer of those arguments there are, the less guarantees we can provide that our agreement will ``diverge'' as we go along, roughly speaking.
Intuitively, this knowledge, as well as the properties we provide above, can help us assess the stability of agreement measures in dynamic scenarios.
\begin{restatable}{reprop}{propaasrelaxedtwoa}\label{proposition-4a}
Let $AAS = (AF, T, SIG)$ and $AAS' = (AF', T', SIG')$ be argumentation-based agreement scenarios s.t.\ $AAS \preceq_N AAS'$, and let $h$ be the $h$-similarity function.
Let us assume that $deg^{h}_{min}(AAS) \geq deg^{h}_{min}(AAS')$.
Let $RMP_{p}$ be a relaxed monotony principle and let it hold true that given $SIG = \langle \sigma_0, ..., \sigma_n \rangle$, for each $\sigma_i, 0 \leq i \leq n$, $\sigma_i$ satisfies $RMP_{p}$.
If for each $\sigma_i$, $\forall E \in \sigma_i(AF)$, $p(AF, AF', E, \sigma_i)$ holds true then the tight upper bound of $\Delta_{\deg^{h}_{min}}(AAS, AAS')$ is $1 - \frac{\lfloor |T'| / 2 \rfloor + |T \cap E_\cap|}{|T'|}$, where %
$E_\cap = \{E_0 \cap ... \cap E_n | E_0 \in \sigma_0(AF), ..., E_n \in \sigma_n(AF), \nexists  E'_0 \in \sigma_0(AF), ..., E'_n \in \sigma_n(AF) \text{ s.t. } |E_0 \cap ... \cap E_n| < |E'_0 \cap ... \cap E'_n|\})$.

\end{restatable}

\section{Degrees of Agreements in Value-based Argumentation}
\label{vba}
%
We extend the formal framework from the previous sections by defining satisfaction and agreement degrees for value-based argumentation. Value-based argumentation allows us to model subjective value preferences for different agents instead of assuming differences between the agents' semantics.
Let us first introduce the notion of a value-based AAS (VAAS).
\begin{definition}[Value-based AAS (VAAS)]
    A value-based agreement scenario is a tuple $(VAF, T, \sigma)$, where $VAF$ is a value-based argumentation framework $VAF = (AR, AT, V, val, {\cal P })$, $T \subseteq AR$ is the \emph{topic}, and $\sigma$ is an argumentation semantics.
\end{definition}
Let us again highlight that in a VAAS, differences between the agents' inference processes are managed via the value preferences that are modeled as part of the VAF. Hence, we do not need to model different argumentation semantics.
We introduce a function that maps value-based AAS to argumentation-based agreement scenarios, so that we can build upon the definitions introduced in the previous section. In the mapping definition, we make use of the notion of a subjective argumentation framework, as introduced in Section~\ref{prelim} as part of the preliminaries for value-based argumentation.
\begin{definition}[AAS-to-Value-based AAS-Mapping]
    $vaas$ is a function that takes a value-based AAS, $VAAS = ((AR, AT, V, val, {\cal P }), T, \sigma)$, and returns an argumentation-based agreement scenario $AAS = ((AR, AT), T, SIG)$ s.t. $SIG = \langle \sigma_0, ..., \sigma_n \rangle$ and each $\sigma_i$, $0 \leq i \leq n$ is an argumentation semantics and the equality $\sigma_i(AF) = \sigma(AF_{P_i})$ holds true.
\end{definition}    
The $vaas$ function allows us to determine degrees of satisfaction and agreement in the same way we do it for AAS.
\begin{definition}[Degrees of Satisfaction and Agreement in Value-based Argumentation]
    Let $VAAS = (VAF, T, \sigma)$ be a value-based AAS, $VAF = (AR, AT, V, val, {\cal P})$, with ${\cal P} = \langle P_0, ..., P_n \rangle$. Let $AAS = (AF, T, SIG)$ be the argumentation-based agreement scenario s.t. $AAS = vaas(VAAS), AAS = ((AR, AT), T, SIG), SIG = \langle  \sigma_0, ..., \sigma_n \rangle$, and let $sim$ be a similarity function.
    We define:
    \begin{itemize}
        \item the degree of minimal agreement of $VAAS$, denoted by $vdeg^{sim}_{min}(VAAS)$, as $deg^{sim}_{min}(AAS)$;
        \item the degree of mean agreement of $VAAS$, denoted by $vdeg^{sim}_{mean}(VAAS)$, as $deg^{sim}_{mean}(AAS)$;
        \item the degree of median agreement of $VAAS$, denoted by $vdeg^{sim}_{med}(VAAS)$, as $deg^{sim}_{med}(AAS)$.
    \end{itemize}
    For $P_i$ and $P_j$ with $0 \leq i \leq n, 0 \leq j \leq n$, we define the two-agent degree of satisfaction between $P_i$ and $P_j$ w.r.t. $VAF$, $T$ denoted by $vsat^{sim}(AF, T, P_i, P_j)$, as $sat^{sim}_{\sigma_i}(AF, T, \sigma_j)$.
\end{definition}

Let us introduce an example of a value-based AAS and its degrees of satisfaction and agreement.
\begin{example}\label{example-vaas-2}
We go back to Example~\ref{example-3} and consider preferred semantics.
We have the value-based AAS $VAAS = ((AR, AT, V, val, {\cal P}), T, \sigma) =  ((\{\arga, \argb, \argc,  \argd\}, \{(\arga, \argb), (\argb, \arga), (\argc, \argb), (\argd, \argc)\}, \{a_v, b_v, c_v, d_v\}, val, \langle P_0, P_1, P_2 \rangle), \{\arga, \argb, \argc, \argd\}, \sigma_{pr})$ s.t.\ $val$ maps every argument $arg \in AR$ to the value $arg_v$ and $\langle P_0, P_1, P_2 \rangle = \langle \{(a_v, b_v)\}, \{(b_v, a_v)\}, \{(c_v, d_v)\} \rangle$. This leaves us with $\sigma_{pr}(AF_{P_0})  = \{\{\arga, \argd\}\}$, $\sigma_{pr}(AF_{P_1}) = \{\{\argb, \argd\}\}$, $\sigma_{pr}(AF_{P_2}) = \{\{\arga, \argc, \argd\}\}$.
Assuming we use the $h$-similarity function (Definition~\ref{degree-sat}), the degrees of satisfaction are provided in Table~\ref{table:4}, and the degrees of agreement are as follows:
\begin{itemize}
    \item The degree of minimal agreement is $\frac{1}{2}$;
    \item The degrees of mean and median agreement are $\frac{3}{4}$.
\end{itemize}
\begin{table*}
\caption{Degrees of satisfaction between agents that are represented by different value preferences in a value-based AAS (Example~\ref{example-vaas-2}).}
\label{table:4}
\centering
\renewcommand{\arraystretch}{1.0}
\begin{tabular}{ |c|c|c|c|} 
 \hline
 \phantom{ } & $\{(a_v, b_v)\}$ & $\{(b_v, a_v)\}$  & $\{(c_v, d_v)\}$\\ 
  \hline
 $\{(a_v, b_v)\}$ & $1$ & $\frac{1}{2}$ & $\frac{3}{4}$ \\ 
 $\{(b_v, a_v)\}$ & $\frac{1}{2}$ & $1$ & $\frac{1}{4}$ \\
 $\{(c_v, d_v)\}$ & $\frac{3}{4}$ & $\frac{1}{4}$ & $1$ \\ 
 \hline
\end{tabular}
\end{table*}
\end{example}
In addition, we introduce a measure of the impact a value has on these notions in a given value-based AAS; this can help agents assess which values are relevant sources of disagreement, so to speak.
\begin{definition}[Value Impact on Degrees of Satisfaction and Agreement]
    Let $VAAS = (VAF, T, \sigma)$ be a value-based AAS, s.t. $VAF = (AR, AT, V, val, {\cal P}), {\cal P} = \langle P_0, ..., P_n \rangle$ and let $sim$ be a similarity function. Let $v \in V$ and let $VAAS'$ be a value-based AAS, s.t. $(VAF', T, \sigma), VAF' = (AR, AT, V, val, {\cal P}')$ and ${\cal P}' = \langle P'_0, ..., P'_n \rangle$, s.t. for $P'_i, 0 \leq i \leq n, P'_i = P_i \cap ((V \setminus v) \times (V \setminus v)$).
    We define:
    \begin{itemize}
        \item The impact of $v$ on $vdeg^{sim}_{min}(VAAS)$, denoted by $imp_v(vdeg^{sim}_{min}(VAAS))$, as \\ $vdeg^{sim}_{min}(VAAS) - vdeg^{sim}_{min}(VAAS')$;
        \item The impact of $v$ on $vdeg^{sim}_{mean}(VAAS)$, denoted by $imp_v(vdeg^{sim}_{mean}(VAAS))$, as \\ $vdeg^{sim}_{mean}(VAAS) - vdeg^{sim}_{mean}(VAAS')$;
        \item The impact of $v$ on $vdeg^{sim}_{med}(VAAS)$, denoted by $imp_v(vdeg^{sim}_{med}(VAAS))$, as \\ $vdeg^{sim}_{med}(VAAS) - vdeg^{sim}_{med}(VAAS')$.
    \end{itemize}
    For $P_i, 0 \leq i \leq n, P_j, 0 \leq j \leq n$, we define the impact of $v$ on $vsat^{sim}(AF, T, P_i, P_j)$, denoted by $imp_v(vsat^{sim}(AF, T, P_i, P_j))$, as $vsat^{sim}_{P_i}(AF, T, P_j) - vsat^{sim}(AF, T, P_i, P'_j)$.
\end{definition}
We continue our previous example to illustrate this notion.
\begin{example}
Considering the value-based AAS $VAAS = ((AR, AT, V, val, {\cal P}), T, \sigma)$ from Example~\ref{example-vaas-2}, we want to determine the impact of the value $b_v$ on the degrees of satisfaction and agreement. Again, we use the $h$-similarity function (Definition~\ref{degree-sat}). We create $VAAS'$, which resembles $VAAS$, but has all preferences that include $b_v$ ``removed'' (roughly speaking), i.e., $VAAS' = ((AR, AT, V, val, {\cal P'}), T, \sigma)$, such that ${\cal P'} = \langle P'_0, P'_1, P'_2 \rangle = \langle \{\}, \{\}, \{(c_v, d_v)\} \rangle$.
To give an example of the value impact on the degree of satisfaction, we can see that $vsat^{h}(AF, T, P'_0, P'_1) = 1$ and hence $imp^{h}_{b_v}(vsat(AF, T, P'_0, P'_1)) = -\frac{1}{2}$.
We have $vdeg^{h}_{min}(VAAS') = \frac{3}{4}$, $vdeg^{h}_{mean}(VAAS') = \frac{11}{12}$ and  $vdeg^{h}_{med}(VAAS') = 1$, and hence  $imp_{b_v}(vdeg^{h}_{min}(VAAS)) = imp_{b_v}(vdeg^{h}_{med}(VAAS)) = -\frac{1}{4}$ and $imp_{b_v}(vdeg^{h}_{mean}(VAAS)) = -\frac{1}{6}$.
\end{example}
For future research, it may be of interest to define more involved approaches to determining the impact of value on the degree of agreement, for example based on game theory-based Shapley values~\cite{shapley1953value}, which are frequently applied in the context of explainable artificial intelligence.

\section{Expanding Value-Based AAS}
\label{exp-vaas}
As the final formal part of this paper, let us extend our notion of AAS expansions, as well as the analysis of how reliable agreements are in the face of expansions, to the case of value-based AAS.
Before we proceed to the formal analysis, we define VAF expansions and normal expansions.
\begin{definition}[VAF (Normal) Expansions]
    Let $VAF = (AR, AT, V, val, {\cal P })$ and $VAF' = (AR', AT', V', val', {\cal P }')$ be value-based argumentation frameworks, such that ${\cal P } = \langle P_0, ..., P_n \rangle$ and ${\cal P }' = \langle P'_0, ..., P'_m \rangle$.
    \begin{itemize}
        \item $VAF'$ is an \emph{expansion} of $VAF$ (denoted by $VAF \preceq_{E} VAF'$) iff it holds true that $(AR, AT) \preceq_E (AR', AT')$, $V \subseteq V'$, $\forall a \in AR$, $val(\arga) = val'(\arga)$,  $|{\cal P }| = |{\cal P }'|$ and for every $P_i, P'_i, 0 \leq i \leq n$, $P_i \subseteq P'_i$.
        \item $VAF'$ is a \emph{normal expansion} of $VAF$ (denoted by $VAF \preceq_{N} VAF'$) iff it holds true that $VAF \preceq_{E} VAF'$, $(AR, AT) \preceq_N (AR', AT')$, $\forall a \in AR$, $val'(\arga) = val(\arga)$, and for every $P_i, P'_i, 0 \leq i \leq n$, for every $(v, v') \in P'_i$ it holds that if $v, v' \in V$ then $(v, v') \in P_i$.
    \end{itemize}
\end{definition}
This allows us to define VAAS normal expansions.
\begin{definition}[VAAS Normal Expansions]
    Let $VAAS = (VAF, T, \sigma)$ and $VAAS' = (VAF', T', \sigma')$ be value-based AAS. $VAAS'$ is a normal expansion of $VAAS$, denoted by $VAAS \preceq_{N} VAAS'$ iff $VAF \preceq_{N} VAF'$, $T \subseteq T', (T' \setminus T) \cap AR = \{\}$ and $\sigma = \sigma'$.
\end{definition}
To illustrate the notion of VAAS normal expansions, let us extend the example of AAS normal expansions.
\begin{example}\label{ex:vaas-normal-exp}
    Let us get back to Example~\ref{ex:aas-expansion} (Figure~\ref{fig:example-exp}).
    Here, we are merely interested in the argumentation frameworks $AF_0 = (AR_0, AT_0) =  (\{\arga, \argb, \argc\}, \{(\arga, \argb), (\argb, \argc)\})$ and $AF_2 = (AR_2, AT_2) = (\{\arga, \argb, \argc,  \argd\}, \{(\arga, \argb), (\argb, \argc), (\argd, \arga)\})$. 
    As we have demonstrated in Example~\ref{ex:aas-expansion}, we cannot use the other argumentation framework  $AF_1 = (\{\arga, \argb, \argc,  \argd\}, \{(\arga, \argb), (\argb, \arga), (\argb, \argc), (\argd, \arga)\})$ and any agreement scenarios based on $AF_1$ to establish an AAS normal expansion relationship to any AAS that is based on $AF_0$, because $AF_0 \preceq_N AF_1$ does not hold.
    Hence, we cannot establish VAAS normal expansion relationships, either, when extending from abstract to value-based argumentation and adjusting the scenarios from AAS to VAAS accordingly.
    Now, let us consider a set of values $V_0 = \{a_v, b_v, c_v\}$, $val_0 = \{(\arga, a_v), (\argb, b_v), (\argc, c_v)\}$, and agents with the following preferences, given $AF_0$:
    \begin{itemize}
        \item $\aga$: $a_v$ is preferred over $b_v$;
        \item $\agb$: $b_v$ is preferred over $a_v$,
    \end{itemize}
    i.e., ${\cal P}_0 = \langle \{(a_v, b_v)\}, \{(b_v, a_v)\} \rangle$, yielding the value-based argumentation framework $VAF_0 = (AR_0, AT_0, V_0, val_0, {\cal P}_0)$.
    Given a topic $T = \{\arga, \argb\}$ and a semantics $\sigma$ (here, we are indifferent to semantics behavior), we can construct the value-based AAS $VAAS_0 = (VAF_0, T, \sigma)$.
    
    For potential expansions, let us consider the following value-based argumentation frameworks and value-based AAS:
    \begin{itemize}
        \item $VAF_2 = (AR_2, AT_2, V_2, val_2,   \langle \{(a_v, b_v)\}, \{(a_v, b_v), (c_v, a_v)\} \rangle$;
        \item $VAF_2' = (AR_2, AT_2, V_2, val_2,   \langle \{(a_v, b_v)\}, \{(b_v, a_v), (d_v, a_v)\} \rangle$;
        \item $VAAS_2  = (VAF_2, T \cup \{\argd\}, \sigma)$;
        \item $VAAS_2'  = (VAF_2', T \cup \{\argc\}, \sigma)$;
        \item $VAAS_2''  = (VAF_2', T \cup \{\argd\}, \sigma)$,
    \end{itemize}
    where $V_2 = V_0 \cup \{d_v\}$ and $val_2 = val_0 \cup \{(\argd, d_v)\}$.

    Now, we can observe the following:
    \begin{itemize}
        \item $VAAS_0 \preceq_N VAAS_2$ does not hold: the second agent's value preferences in $VAF_2$ are inconsistent with the second agent's value preferences in $VAF_0$;
        \item $VAAS_0 \preceq_N VAAS_2'$ does not hold: the topic is expanded by adding an argument from the initial argumentation framework; 
        \item $VAAS_0 \preceq_N VAAS_2''$ holds: the topic is expanded by adding a ``new'' argument, and the agents' value preferences are consistent.
    \end{itemize}
\end{example}
Let us introduce the notion of a (value-based) agreement delta also for value-based AAS.
\begin{definition}[Value-based Agreement Delta]\label{def:vaas-delta}
   Let $VAAS$ and $VAAS'$ be value-based AAS and let $sim$ be a similarity function. We define the value-based agreement delta of $VAAS$ and $VAAS'$ w.r.t. $sim$ and a degree or agreement $vdeg^{sim} \in \{vdeg^{sim}_{min}, vdeg^{sim}_{mean}, vdeg^{sim}_{med}\}$, denoted by $\Delta_{vdeg^{sim}}(VAAS, VAAS')$, as $|vdeg^{sim}(VAAS) - vdeg^{sim}(VAAS')|$.
\end{definition}
We can illustrate the notion of a value-based agreement delta by revisiting one of our examples.
\begin{example}
\label{ex:delta-vaas}
Consider the following VAFs.
\begin{itemize}
    \item $VAF = (AR, AT, V, val, \langle P_0, P_1 \rangle)$, where $AR = \{\arga, \argb, \argc \}$, $AT = \{(\arga, \argb), (\argb, \arga), (\argb, \argc)\}$, $V = \{a_v, b_v, c_v\}$, $val = \{(\arga, a_v), (\argb, b_v), (\argc, c_v)\}$, $P_0 = \{(a_v, b_v)\}$, and $P_1 = \{(b_v, a_v)\}$;
    \item $VAF' = (AR', AT', V', val', \langle P'_0, P'_1 \rangle$, where $AR' = \{\arga, \argb, \argc,  \argd\}$, $AT' = \{(\arga, \argb), (\argb, \arga), (\argb, \argc), (\argd, \arga)\})$, $V' = V \cup \{d_v\}$, $val' = val \cup \{(\argd, d_v)\}$, $P'_0 = \{(a_v, b_v)\}$, and $P'_1 = \{(b_v, a_v), (d_v, a_v)\}$.
\end{itemize}
Note that the VAFs are similar to $VAF_0$ and $VAF'_2$ in Example~\ref{ex:vaas-normal-exp}, but feature an additional attack from $\argb$ to $\arga$.
Given these VAFs, we define the following VAAS, given the topic $T = \{\arga, \argb\}$:
\begin{itemize}
    \item $VAAS = (VAF, T, \sigma_{pr})$;
    \item $VAAS' = (VAF', T \cup \{\argd\}, \sigma_{pr})$.
\end{itemize}
We observe that $\sigma_{P_0}(VAF) = \{\{\arga, \argc\}\}$ and $\sigma_{P_1}(VAF) = \{\{\argb\}\}$.
Accordingly, $vdeg^{sim}_{min}(VAF) = vdeg^{sim}_{mean}(VAF) = vdeg^{sim}_{med}(VAF) = 0.5$.
In contrast, $\sigma_{P'_0}(VAF') = \sigma_{P'_1}(VAF') =  \{\argb, \argd\}$ and hence $vdeg^{sim}_{min}(VAF') = vdeg^{sim}_{mean}(VAF') = vdeg^{sim}_{med}(VAF') = 1$.
This give us the value-based agreement deltas: $\Delta_{vdeg^{sim}_{min}}(VAAS, VAAS')  = \Delta_{vdeg^{sim}_{mean}}(VAAS, VAAS') = \Delta_{vdeg^{sim}_{med}}(VAAS, VAAS') = |0.5 - 1| = 0.5$.
\end{example}
Somewhat analogously to our analysis of AAS, we can show that the weak cautious monotony principle $RMP_{cm}$ guarantees the following, two value-based AAS $VAAS = (VAF, T, \sigma), VAF = (AR, AT, V, val, {\cal P }), {\cal P } = \langle P_0, ..., P_n \rangle$ and $VAAS' = (VAF', T', \sigma')$, s.t. $VAAS \preceq_N VAAS'$:
\begin{enumerate}
    \item If $T = T'$ and all agents infer $T$ as a subset of at least one of their (subjective) extensions, then we retain a maximal degree of (minimal, median, and mean) agreement;
    \item  Again, the stability guarantees of the degree of minimal agreement depend on the relative number of topic arguments that all agents agree on inferring, as well as on the stability of the topic.
\end{enumerate}

Let us first prove a preliminary: we want to characterize value-based AAS as ``ordinary'' AAS, in order to utilize the results obtained in the previous section.
\begin{restatable}{reprop}{prelimprop}\label{prelim-prop}
Let $VAAS = (VAF, T, \sigma)$, with $VAF = (AR, AT, V, val, {\cal P })$ and ${\cal P } = \langle P_0, ..., P_n \rangle$, and $VAAS' = (VAF', T', \sigma')$, with $VAF' = (AR', AT', V', val', {\cal P }')$ and $\langle P'_0, ..., P'_n \rangle$, be value-based AAS s.t.\ $VAAS \preceq_N VAAS'$. Given that $\sigma$ satisfies weak cautious monotony, it holds true that there exist two argumentation-based agreement scenarios $AAS$ and $AAS'$ s.t.\ $vdeg^{sim}_{min}(VAAS) = deg^{sim}_{min}(AAS)$ and $vdeg^{sim}_{min}(VAAS') = deg^{sim}_{min}(AAS')$, $AAS = (AF, T, SIG)$, $AAS' = (AF', T', SIG)$, $AAS \preceq_N AAS'$, $SIG = \langle \sigma_0, ..., \sigma_n \rangle$, and each $\sigma_i, 0 \leq i \leq n$, is an argumentation semantics that satisfies weak cautious monotony.
\end{restatable}
Now, we can provide the main propositions.
\begin{restatable}{reprop}{vaaspropone}
Let $VAAS = (VAF, T, \sigma)$, with $VAF = (AR, AT, V, val, {\cal P })$ and $VAAS' = (VAF', T', \sigma')$, with $VAF' = (AR', AT', V', val', {\cal P }')$, be value-based AAS, s.t. $VAAS \preceq_N VAAS'$. Let ${\cal P } = \langle P_0, ..., P_n \rangle, {\cal P' } = \langle P'_0, ..., P'_n \rangle$, let $T = T'$, and let $h$ be the $h$-similarity function.
Let $\sigma$ be an argumentation semantics that satisfies weak cautious monotony.
If for $0 \leq i \leq n$, $\exists E \in \sigma(AF_i)$ s.t. $T \subseteq E$ and $cm(AF_{P_i}, AF'_{P'_i}, E, \sigma)$ holds true then $vdeg^{h}_{min}(AAS) = vdeg^{h}_{min}(AAS') = vdeg^{h}_{mean}(AAS) = vdeg^{h}_{mean}(AAS') = vdeg^{h}_{med}(AAS) = vdeg^{h}_{med}(AAS') = 1$.
\end{restatable}
We show that given two value-based AAS $VAAS = (VAF, T, SIG)$ and $VAAS' = (VAF', T', SIG')$, such that $VAAS \preceq_N VAAS'$, $1 - \frac{\lfloor |T'| / 2 \rfloor + |T \cap E_\cap|}{|T'|}$ is the least upper bound of  $|vdeg^{h}_{min}(VAAS) - vdeg^{h}_{min}(VAAS')|$, if $\sigma$ satisfies the weak cautious monotony principle $RMP_{cm}$ and $cm(AF, AF', E, \sigma)$ holds true, where $E_\cap$ is the maximal intersection of the subjective extensions the agents infer from the initial value-based AAS. Again (as in Section~\ref{exp-aas}), we assume that $|T| > 0$ and that the degree of minimal agreement decreases with the expansion from $VAAS$ to $VAAS'$.
\begin{restatable}{reprop}{vaasproptwo}
Let $VAAS = (VAF, T, \sigma), VAF = (AR, AT, V, val, {\cal P })$ and $VAAS' = (VAF', T', \sigma'), VAF' = (AR', AT', V', val', {\cal P }')$ be argumentation-based agreement scenarios s.t.\ $VAAS \preceq_N VAAS'$. Let ${\cal P } = \langle P_0, ..., P_n \rangle, {\cal P' } = \langle P'_0, ..., P'_n \rangle$ and let $h$ be the $h$-similarity function. Let us assume that $vdeg^{h}_{min}(VAAS) \geq vdeg^{h}_{min}(VAAS')$.
Let $\sigma$ be an argumentation semantics that satisfies weak cautious monotony.
If for $0 \leq i \leq n$, $cm(AF_{P_i}, AF'_{P'_i}, E, \sigma_i)$ holds true then the least upper bound of $\Delta_{vdeg^{h}_{min}}(VAAS, VAAS')$ is $1 - \frac{\lfloor |T'| / 2 \rfloor + |T \cap E_\cap|}{|T'|}$, where $E_\cap = \{E_0 \cap ... \cap E_n | E_0 \in \sigma_0(AF), ..., E_n \in \sigma_n(AF), \nexists  E'_0 \in \sigma_0(AF), ..., E'_n \in \sigma_n(AF) \text{ s.t. } |E_0 \cap ... \cap E_n| < |E'_0 \cap ... \cap E'_n|\})$.
\end{restatable}
Let us highlight that in contrast to the more general propositions in Section~\ref{exp-aas}, we can provide the above guarantees merely in case of cautious monotony and not for any relaxed monotony principle.
The reason for this is that we must rely on the specifics of the $p$-function of cautious monotony that intuitively relaxes monotony only in face of new direct attackers: in value-based agreement scenarios, such attackers cannot exist in subjective argumentation frameworks if they do not also exist in the corresponding (abstract) agreement scenario representation.

\section{Implementation and Experiments}
\label{sec:implementation}
As an initial step towards an empirical perspective on our formal approach to degrees of agreement in argumentation dialogues, we provide a software implementation that
supports the specification of argumentation-based agreement scenarios (value-based or not) and the computation of degrees of satisfaction and agreement, as well as of the impact values have on degrees of agreement.
The implementation is an extension of the \emph{Diarg} argumentation-based dialogue reasoner~\cite{kampik2020diarg}, which in turn is based on \emph{Tweety}, a well-known library for argumentation and defeasible reasoning~\cite{10.5555/3031929.3031992}.
Like Tweety, our implementation is provided in Java, a mainstream high-level programming language.
A tutorial on how to work with our implementation is provided at \url{http://s.cs.umu.se/mhfrcp}, alongside the source code and additional documentation and tests.

Using our implementation, we can conduct initial experiments based on synthetically generated argumentation frameworks to obtain some empirically informed intuitions.
Below, we do exactly this, focusing on value-based AAS and the following two questions:
\begin{enumerate}
    \item Given a value-based AAS and a normal expansion of it, how large is the delta between initial and final degrees of (minimal, median, and median) agreement,  and how is this delta affected by the size of the expansion (in terms of number of new arguments)?
    \item Given a value-based AAS, how large is the impact a value has on the degrees of (minimal, median, and mean) agreement, and how does the size of the argumentation framework (in terms of number of arguments) affect the impact?
\end{enumerate}
We focus on value-based AAS and do not cover ``ordinary'' AAS because the former explicitly model subjective differences in agent preferences that affect the extensions inferred by the agents, whereas in the latter case, the degrees of agreement are affected solely by differences in how semantics treat topological properties (arguably: often nuances) of argumentation frameworks.
Hence, we consider an empirical study of value-based AAS more interesting for gauging future application potential.
In contrast, a study of AAS could, if conducted thoroughly, produce empirical insights regarding the nature of differences between argumentation semantics, but is considered out-of-scope in the context of our paper.

In order to shed light on Question 1, we first generate an initial value-based AAS and determine the degrees of agreement (Step 1). Then, we expand the value-based AAS and again determine the degrees of agreement and their delta to the previously computed degrees (Step 2), as described in more detail below:
\begin{enumerate}
    \item Each argument has up to three attack targets, and each of the three attacks (to any random argument) is generated with a probability of 50\%. Self-attacks are excluded. Each argument is mapped to a value (we can then say that we have \emph{single-argument values} and hence our VAFs are, intuitively, \emph{multi-agent preference-based argumentation frameworks}). A topic is generated that includes any of the arguments with 50\% probability. We then generate five value preferences for each of the three agents, in a pseudo-random manner, considering the following constraints:
    \begin{enumerate*}[label=\roman*)]\label{en:vp-constraints}
    \item each value preference runs contrary to the direction of an attack between two arguments that are mapped to the corresponding values, in order to ensure the value preference can have an actual impact on the inferences that are drawn; 
    \item an agent's value preferences must be transitive.
    \end{enumerate*}
     Then, we instantiate the initial value-based AAS and determine its degrees of (minimal, mean, and median) agreement.
    \item We then expand the (value-based) argumentation framework by adding one or several arguments (see below) and attacks from these arguments to any other arguments, but not vice versa\footnote{I.e., we have a \emph{strong expansion}~\cite{baumann2010expanding}.}. We also generate a new value for each of the new arguments and add new value preferences for each agent, approximately one for each new argument, again in a pseudo-random manner considering the constraints above (preferences running contrary to attacks and transitive preferences per agent).
    Then, we generate a new value-based AAS with the expanded value-based argumentation framework and again determine the new degrees of agreement, which we then use to compute the (absolute) delta between then initial and new degrees (minimal, median, and mean). When generating the new value-based AAS, we run experiments for two configurations:
    \begin{enumerate*}[label=\roman*)]
        \item in one configuration, we keep the topic fixed, i.e., no new arguments are added to the topic;
        \item in the other configuration, we expand the topic by adding each ``new'' argument to the topic with a probability of 50\%.
    \end{enumerate*}

\end{enumerate}
We follow this procedure for an initial argumentation framework with between one and ten arguments (determined at random) and an expansion that contains between one and 15 additional arguments (repeatedly, increasing by increments of one). For each expansion size, we create a new initial argumentation framework and expansion 30 times and we average the computed changes to the degrees of agreement (separated by minimal, median, and mean) to then plot the obtained values per expansion size for fixed and expanding topic sizes (Figures~\ref{fig:agreement-degrees-wo-topic-exp} and~\ref{fig:agreement-degrees-w-topic-exp})\footnote{Changes to the degrees of agreement are always determined relative to an initial argumentation framework and not relative to the next smaller expansion.}.
In addition, we provide plots (in Figures~\ref{fig:agreement-degrees-wo-topic-exp-n} and~\ref{fig:agreement-degrees-w-topic-exp-n}) showing the average change in a \emph{normalized} manner so that it is bounded by the maximum distance to the bounds of $0$ and $1$ (and not by 1). For example, if the degree of agreement is 0.5 to start with, it can (without further assumptions about the properties of the measure) maximally change by 0.5 (to either 0 or 1), and we have added plots that reflect this fact, thus painting a more cautious picture.

The most obvious observations we can make are that
\begin{enumerate*}[label=\roman*)]
    \item degrees of median agreement are the most stable and degrees of minimal agreement the least, as informally claimed in Section~\ref{sec:argumentation};
    \item expanding the topic leads to greater instability---in particular when considering the change in degrees of minimal agreement, normalized relative to what would be the maximal change possible, with expanding topics, we see the potential for rapid changes.
\end{enumerate*}
For the latter observations, note the difference in scale between the y-axes of Figures~\ref{fig:agreement-degrees-wo-topic-exp} and~\ref{fig:agreement-degrees-w-topic-exp}, as well as of Figures~\ref{fig:agreement-degrees-wo-topic-exp-n} and~\ref{fig:agreement-degrees-w-topic-exp-n}.

We cannot observe a clear increase in the delta measuring the change in degrees of minimal and mean agreement with an increasing number of arguments in the expansion; considering the initial spikes in Figures~\ref{fig:agreement-degrees-wo-topic-exp} and~\ref{fig:agreement-degrees-wo-topic-exp-n} (representing the same experiments at different scales), we can speculate (without certainty) that the degree of agreement delta is not closely associated with expansion size.
To the contrary, given a fixed topic, an increasing expansion size may even imply that the topic set of the initial (value-based) agreement scenario is \emph{less likely} subjected to substantial changes to the degrees of agreement.
For expanding topics we see analogous spikes for larger expansions in Figures~\ref{fig:agreement-degrees-w-topic-exp} and~\ref{fig:agreement-degrees-w-topic-exp-n}, which is a further indication that the spikes in Figures~\ref{fig:agreement-degrees-wo-topic-exp} and~\ref{fig:agreement-degrees-wo-topic-exp-n} should be interpreted with caution. 
All in all, the plots indicate that a ``disagree and commit'' approach may be viable, as in particular the changes to degrees of median agreement are relatively low given non-expanding topic sizes and our synthetically generated data.
However, the degree of minimal agreement may be impacted substantially, in particular for expanding topic sizes. This is most pronounced when considering actual versus maximal change in the degree of minimal agreement; i.e., it cannot be assumed that the satisfaction of all agents remains generally stable.

\begin{figure}[ht]
\centering
\subfloat[Fixed topic.]{
\label{fig:agreement-degrees-wo-topic-exp}
\includegraphics[width=0.45\columnwidth]{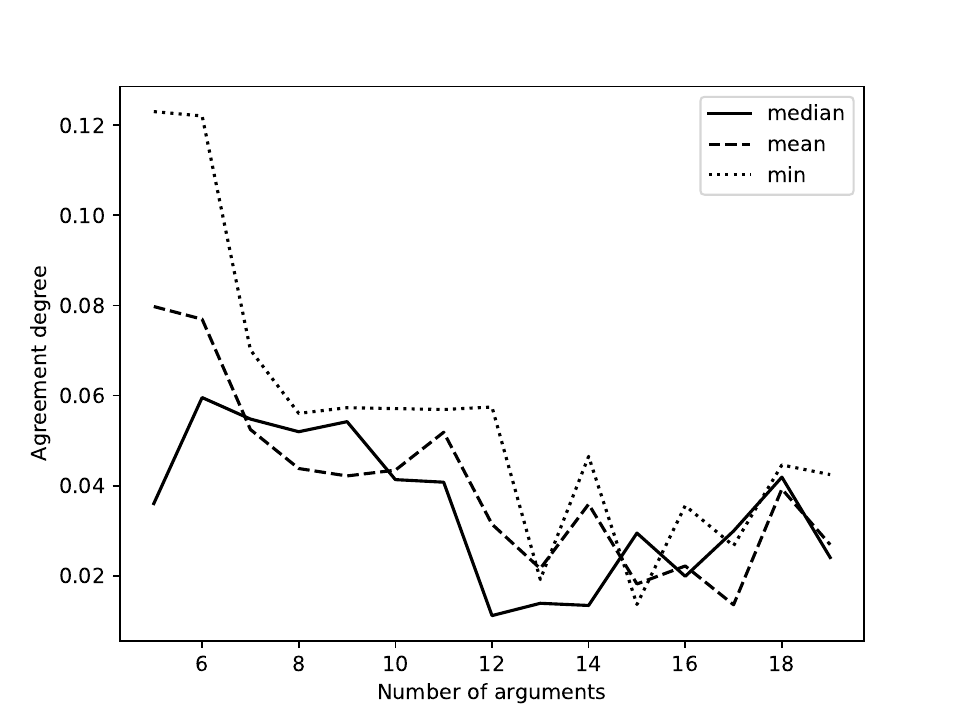}
}
\subfloat[Expanding topic.]{
\label{fig:agreement-degrees-w-topic-exp}
\includegraphics[width=0.45\columnwidth]{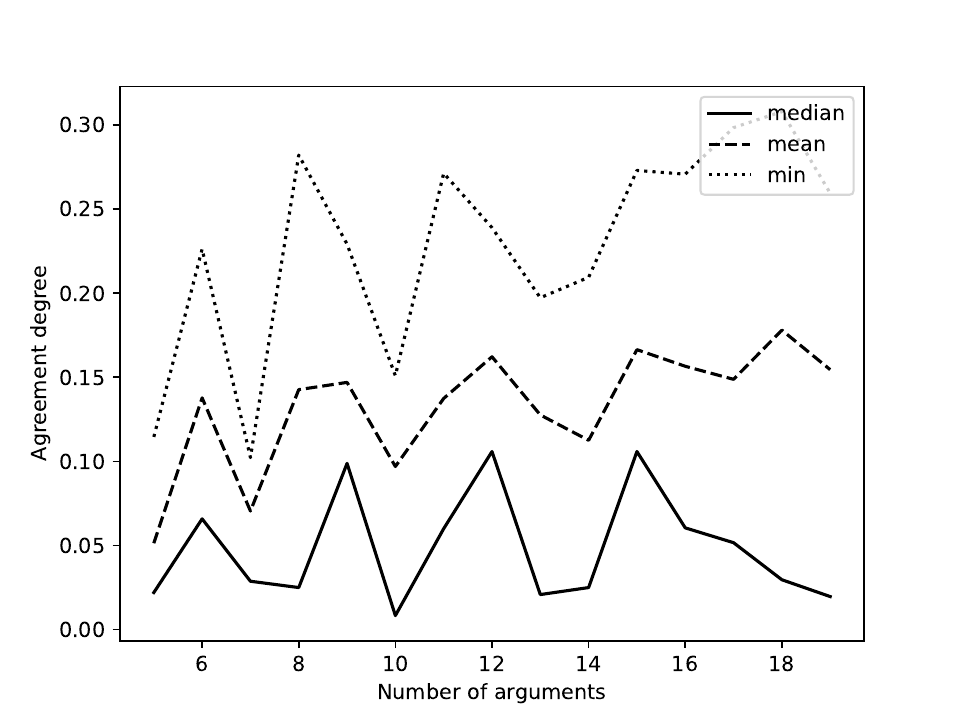}}

\subfloat[Fixed topic, normalized degrees of agreement.]{
\label{fig:agreement-degrees-wo-topic-exp-n}
\includegraphics[width=0.45\columnwidth]{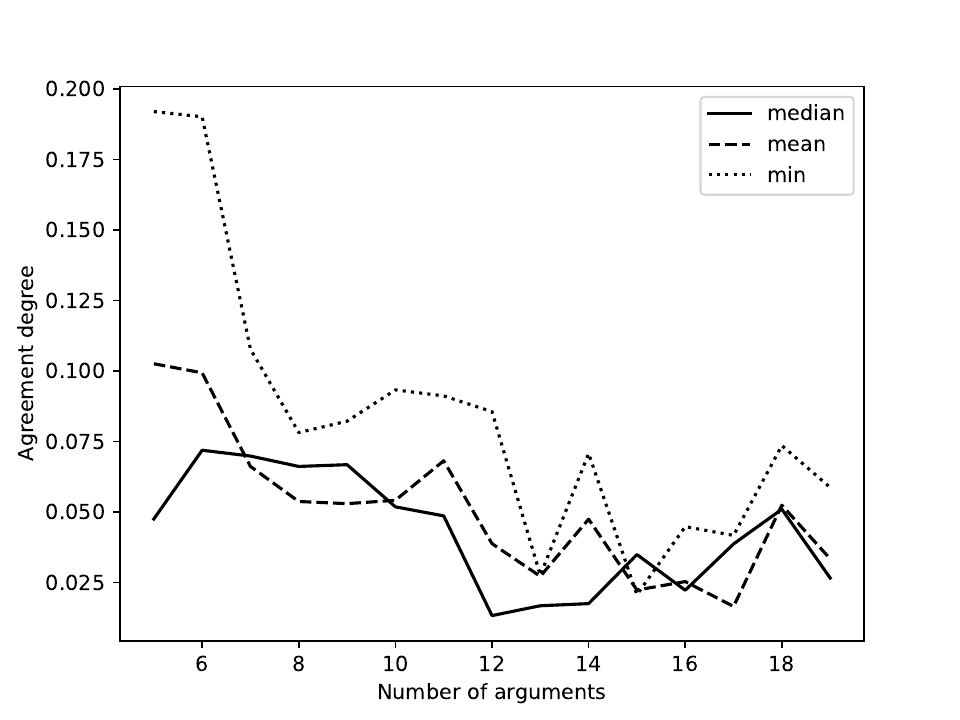}
}
\subfloat[Expanding topic, normalized degrees of agreement.]{
\label{fig:agreement-degrees-w-topic-exp-n}
\includegraphics[width=0.45\columnwidth]{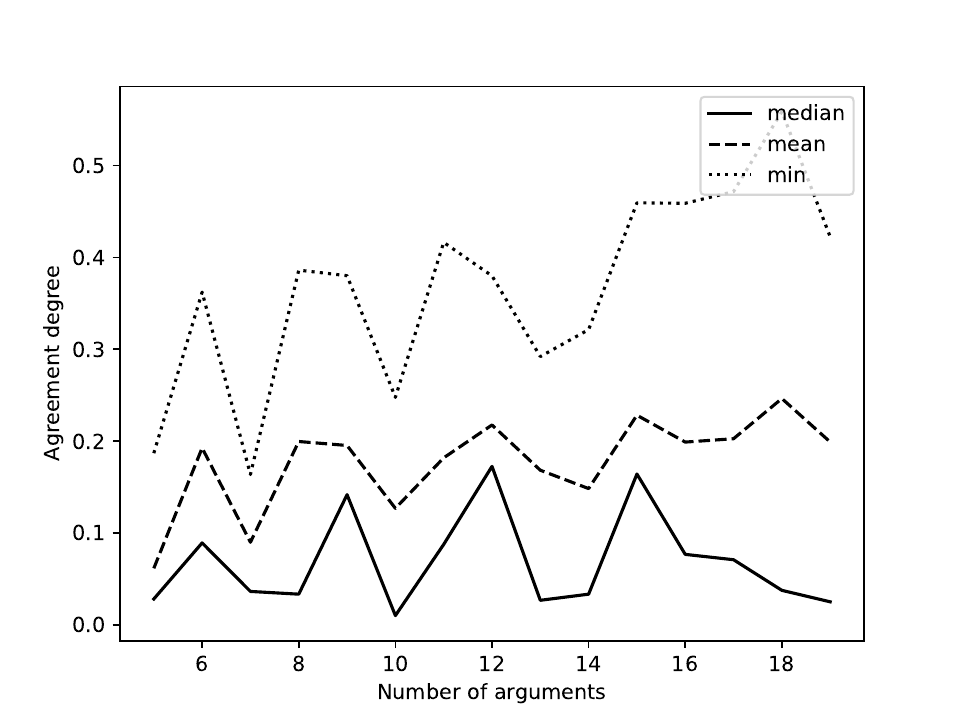}}
\caption{Delta between initial and final degrees of agreement after increasingly large (normal) expansions.}
\label{fig:agreement-degrees}
\end{figure}

In order to shed light on Question 2, we generate value-based AAS with three agents and between 5 and 20 arguments, following the same approach as the one we use in the first step of the previous experiment (without determining degrees of agreement in this case).
However, we diverge from the previous initial generation step regarding the generation of the set of topic arguments.
More specifically, in the \emph{fixed topic} setting, a topic is generated that includes any of the first five arguments with 50\% probability, and, given the configuration in which we have a topic size proportional to the number of arguments in our argumentation framework, any argument is included in the topic with 50\% probability.

For each argumentation framework size in $[5, 20]$, we generate 30 argumentation frameworks in the way described above and randomly select a value that we consider ``relevant'', i.e., that is mapped to an argument that is the target of an attack.
We then determine the absolute impact of this value on the degree of minimal, median, and mean agreement and average the sum (ignoring the sign of the impact) over the 20 argumentation frameworks for a given size.
The resulting plots for fixed and expanding topic sizes are shown in Figure~\ref{fig:value-impact}, again plotting both the absolute change of the degrees of agreement (Figures~\ref{fig:value-impact-wo-topic-exp} and~\ref{fig:value-impact-w-topic-exp}), as well as the normalized change relative to the maximal change that can possibly occur (Figures~\ref{fig:value-impact-wo-topic-exp-n} and~\ref{fig:value-impact-w-topic-exp-n}).
What we can see is that the impact of a single-argument value tends to be lower with a larger number of arguments in an argumentation framework, \emph{given the topic size is proportional to the number of arguments} (Figures~\ref{fig:value-impact-w-topic-exp} and~\ref{fig:value-impact-w-topic-exp-n}). This is less clear when the topic size is independent of the number of arguments in the argumentation framework (Figures~\ref{fig:value-impact-wo-topic-exp} and~\ref{fig:value-impact-wo-topic-exp-n}).
Generally, we can see that the impact of (single-argument) values on degrees of agreement is moderate: even in the case of degrees of minimal agreement, the impact is just a small fraction (typically not more than 10\% and, in the averaged experiments, never more than 20\%) of the maximal impact a value could potentially have (Figures~\ref{fig:value-impact-wo-topic-exp-n} and~\ref{fig:value-impact-w-topic-exp-n}).

\begin{figure}[ht]
\centering
\subfloat[Fixed topic.]{
\label{fig:value-impact-wo-topic-exp}
\includegraphics[width=0.45\columnwidth]{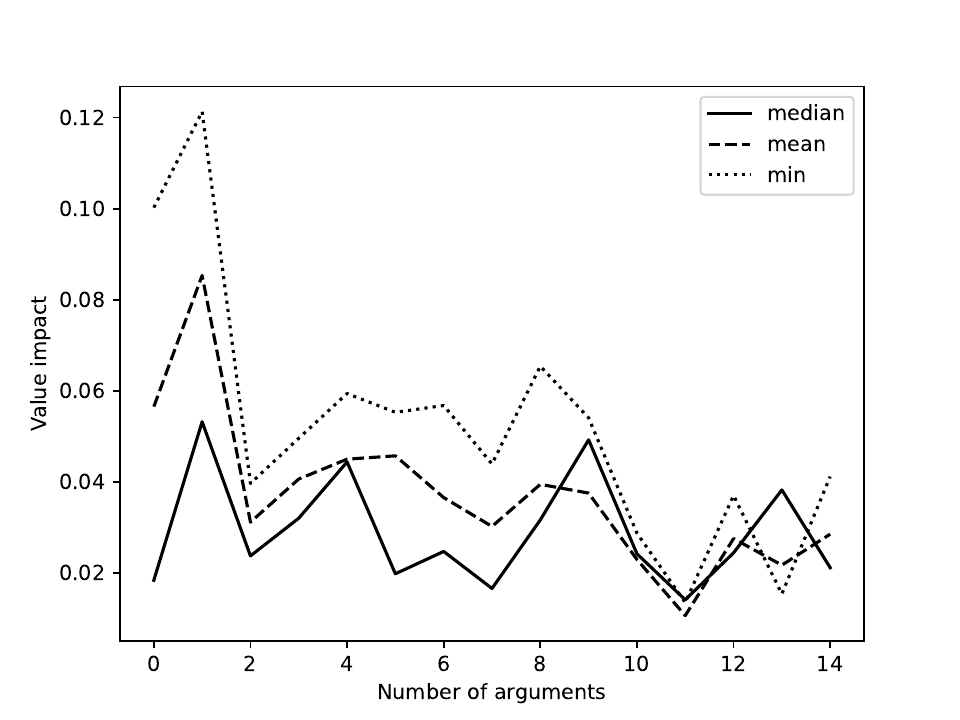}
}
\subfloat[Expanding topic.]{
\label{fig:value-impact-w-topic-exp}
\includegraphics[width=0.45\columnwidth]{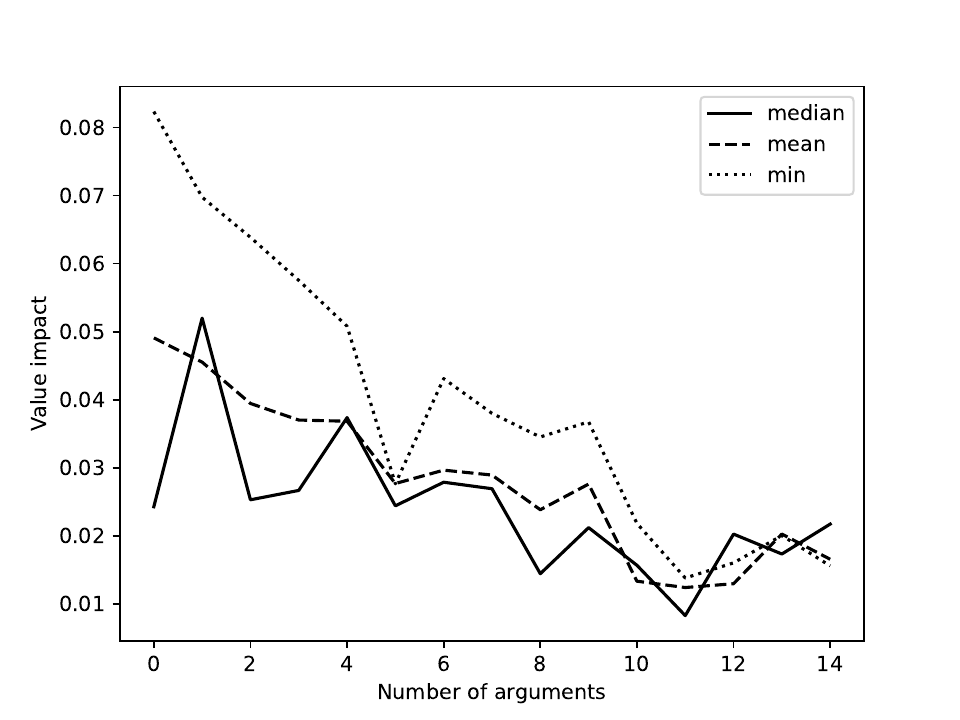}} \\
\subfloat[Fixed topic, normalized degrees of agreement.]{
\label{fig:value-impact-wo-topic-exp-n}
\includegraphics[width=0.45\columnwidth]{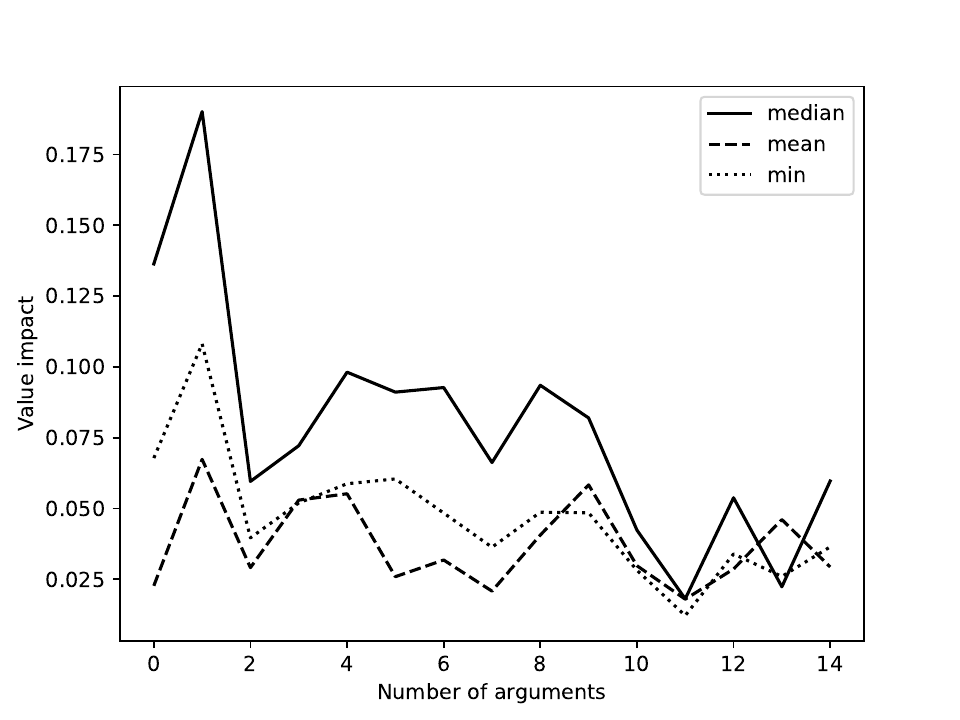}
}
\subfloat[Expanding topic, normalized degrees of agreement.]{
\label{fig:value-impact-w-topic-exp-n}
\includegraphics[width=0.45\columnwidth]{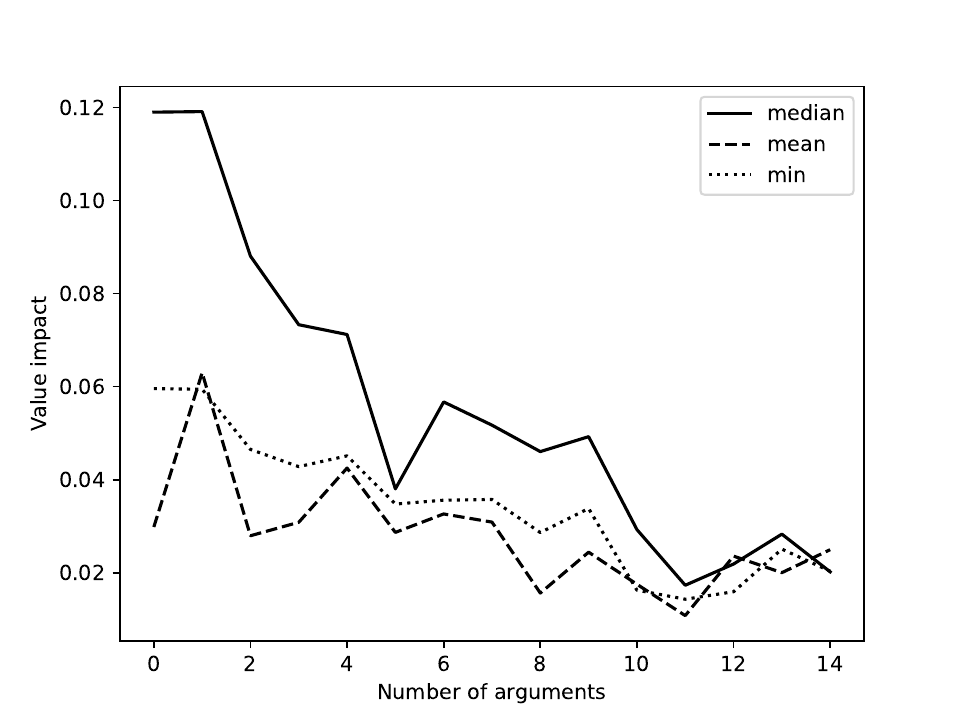}}
\caption{Value impact of random ``relevant'' (i.e., attack-overriding) values in argumentation frameworks of different sizes.}
\label{fig:value-impact}
\end{figure} 
The code generating the synthetic data (albeit in an non-deterministic manner) is available at \url{http://s.cs.umu.se/rcapc1}.
Let us highlight that the experiments presented above are merely initial steps towards providing empirical perspectives on the formal results at the core of the paper; e.g., one of the many design choices affecting the evaluation in a non-trivial manner is the focus on 1-to-1 mappings between arguments and values.
More experiments on synthetic and ideally real-world data are necessary to provide a holistic understanding, which we strongly encourage for future research.

\section{Related Research}
\label{related}
The research presented in this paper takes the idea of degrees of agreement, first developed for a specific variant of structured argumentation~\cite{nieves2019approximating}, and lifts them to the abstract level; the two approaches are not directly comparable, as they assume fundamentally different underlying formal models. Let us argue that the approach taken in this paper is appealing because it takes abstract argumentation as a starting point, which is a very simple model of reasoning in face of conflicts that is well-studied and widely understood within the broader symbolic AI community.

The issue our theoretical framework and its software implementation tackle is of a different nature than the problems investigated in works on merging knowledge bases in general~\cite{10.1007/978-3-030-00461-3_15} and argumentation frameworks in particular~\cite{10.5555/3032027.3032032} in that we focus on the result of an inference process and not on the input knowledge base.
In contrast to works on argumentation and judgment aggregation (see~\cite{bodanza-survey} for a survey), our work does not primarily address the problem of determining a joint judgment, but rather of assessing to what extent diverging judgments are aligned and may stay aligned in a dynamic environment; thereby, we provide a novel bridge between multi-agent argumentation and \emph{argumentation dynamics} (see~\cite{doutre-argument} for a survey).
In contrast to other works on argumentation dynamics such as~\cite{Cayrol2010,Bisquert2011} that formalize change operations (for adding and removing an argument, respectively), we instead model change as the normal expansion of an argumentation framework.
This means we focus only on the addition of new arguments. Let us claim that this is a well-motivated constraint: after all, the removal of arguments can be modeled by introducing dummy \emph{annihilator arguments} (cf.~\cite{10.1093/logcom/exu007}).
Then, the change of an argument's outgoing attacks can be modeled by its removal and replacement by a new argument.

Also, our work relates to the line of research of enforcement in formal argumentation as surveyed in~\cite{baumann2021enforcement}.
Enforcement studies typically focus on the (minimal) changes that are necessary to ensure the acceptance (or rejection) of arguments (notably, \cite{baumann2012does,ulbricht2019if}\footnote{Note that in~\cite{ulbricht2019if}, the objective is to enforce the acceptance of a non-empty set of arguments, i.e., enforcement is moved to the meta-level.}).
In contrast, we focus on the enforcement of argumentation principles---in particular of systematic relaxations of monotony of entailment, thereby moving enforcement to the level of argumentation semantics.
Monotony and systematic relaxations thereof have been extensively studied in the context of formal argumentation, e.g., in~\cite{10.1007/978-3-319-28460-6_6,10.1007/978-3-319-28460-6_13,baumann2010expanding}.

Finally, let us highlight that some of the key notions introduced in this paper---namely degrees of satisfaction and agreement---can be applied to any choice scenario where agreements between agents that may have more than one preferred choice option need to be measured; i.e., these notions are somewhat agnostic to formal argumentation, although the idea of \emph{extensions} in the sense of several possible and equally valid inferences plays a central role. Here, future work may further position these notions in the context of social choice research~\cite{10.5555/3180776}.


\section{Conclusion and Future Research}
\label{questions}
This paper provides a formal framework for degrees of agreements in abstract and value-based argumentation dialogues, as well as guarantees with respect to bounds for changes in degrees of agreement, given principle-based constraints. 
A possible next research step is to assess the computational complexity of determining degrees of satisfaction and agreements, as well as of enforcing relaxed monotony principles; determining degrees of agreement can be computationally costly, as it requires a search through the powerset of an argumentation framework's arguments. 
Besides this, a range of research directions can be considered promising to further develop the introduced approach. In the theoretical analysis, for the sake of conciseness, the proofs of bounds of changes in degrees of agreement in argumentation-based agreement scenarios are limited to the \emph{degree of minimal agreement}. An analysis of maximal changes in the degrees of mean and median may be relevant, although intuitively, the \emph{minimal} bound is of greatest interest.
More fundamentally, the research presented in this paper could be extended to align with a novel approach to abstract argumentation in which the semantics order sets of arguments according to their \emph{plausibility}~\cite{skiba2020towards}. These semantics
naturally fit the argumentation-based agreement scenarios this paper introduces, as new measures of satisfaction can be straightforwardly derived from the plausibility orders.
A key limitation of our research is the assumption that the agents do not act strategically, i.e., we do not model an agent's preferences over choice options given the preferences of other agents. To achieve this, future work can build upon results in game theory on norm-based equilibria~\cite{RICHTER2020105042}. There are additional lines of work our approach can be integrated with, e.g., a preference-based argumentation method for multi-criteria decision-making~\cite{AMGOUD2009413}, which presents an application of formal argumentation to a similar problem, but does not cover degrees of agreement and consistency/relaxed monotony principles, and argumentation context frameworks~\cite{10.1007/978-3-642-04238-6_7}, which are, roughly speaking, further extending value-based argumentation frameworks to support distributed argumentation.
Crucially, we want to emphasize the need to expand on the initial empirical perspectives that we provide in Section~\ref{sec:implementation}, ideally utilizing real-world data.
Finally, we consider the assessment of the impact of values on collaborative inferences and decisions a subject that invites further study in real-world social contexts.

\backmatter

\bmhead{Acknowledgments}
We thank the anonymous reviewers for very detailed feedback that has helped us to substantially improve the presentation of the paper.
This work was partially supported by the Wallenberg AI, Autonomous Systems and Software Program (WASP) funded by the Knut and Alice Wallenberg Foundation.

\begin{appendices}
\section*{Appendix: Proofs}
The appendix re-states all propositions that are presented in the body of the paper and provides their proofs.
%
\wcmprop*
%
\begin{proof}
\phantom{e}
\begin{enumerate}
    \item By definition, weak cautious monotony is satisfied iff the following statement holds true for every two argumentation frameworks $AF = (AR, AT), AF' = (AR', AT'), AF \preceq_N AF'$.
\begin{align*}
&{} \forall E \in \sigma(AF), \\
&{} \text{if } \{(\arga,\argb) \mid (\arga,\argb) \in AT', a \in AR' \setminus AR, b \in E \}= \{\} \\
&{} \text{then } \exists E' \in \sigma(AF') \text{ s.t. } E \subseteq E'
\end{align*}
\item According to the definition of a relaxed monotony principle, $RMP_{cm}$ is satisfied by an argumentation semantics $\sigma$ iff the following statement holds true for every two argumentation frameworks $AF = (AR, AT), AF' = (AR', AT'), AF \preceq_N AF'$:
\begin{align*}
    &{} \forall E \in \sigma(AF), \\
    &{} {if } \: cm(AF, AF', E, \sigma) \text{ holds true} \\ &{} 
    \text{ then } \exists E' \in \sigma(AF'),  \text{ s.t. } E \subseteq E'
\end{align*}
\item From \emph{2.} it follows that by definition of $RMP_{cm}$, $\sigma$ violates $RMP_{cm}$ iff there exists two argumentation frameworks $AF = (AR, AT), AF' = (AR', AT')$, such that $ AF \preceq_N AF'$ and the following statement does not hold true:
\begin{align*}
    &{} \forall E \in \sigma(AF), \\
    &{} {if }  \{(\arga, \argb) |(\arga, \argb) \in AT', \arga \in AR' \setminus AR, \argb \in E \} = \{\}
    \\ &{}  \text{ then } \exists E' \in \sigma(AF'),  \text{ s.t. } E \subseteq E'
\end{align*}
Clearly, this statement (following the \emph{iff}) must be false if \emph{1.} holds, and if the statement holds, {1.} must be false. Hence, $\sigma$ violates $RMP_{cm}$ iff weak cautious monotony is violated, which proves the proposition.
\end{enumerate}
\end{proof}
%
\rmpprop*
%
\begin{proof}
\phantom{e}
\begin{enumerate}
    \item By definition of $\sigma_{na}$, for every conflict-free set $S \subseteq AR'$ it holds true that $\exists E' \in \sigma_{na}(AF')$, s.t. $S \subseteq E'$.
    \item By definition of $\sigma_{na}$, $\forall E \in \sigma_{na}(AF)$, it holds true that $E$ is a conflict-free set in $AF$.
    \item Because $AF \preceq_N AF'$, it follows from 2. that $\forall E \in \sigma_{na}(AF)$, it holds true that $E$ is a conflict-free set in $AF'$.
    \item Hence, it follows from 1. that $\forall E \in \sigma_{na}(AF)$, $\exists E' \in \sigma_{na}(AF')$, s.t.  $E \subseteq E'$. We have proven the proposition.
\end{enumerate}
\end{proof}
%

\propaasrelaxedonea*
%
\begin{proof}
Because for $\sigma_i$, $0 \leq i \leq n$, it holds that $\exists E \in \sigma_i(AF)$ s.t. $T \subseteq E$, we know, given the degrees of agreement definitions (Definition~\ref{degs-aas}), that $\deg^{h}_{min}(AAS) = \deg^{h}_{mean}(AAS) = \deg^{h}_{med}(AAS) = 1$ (our agents can fully agree on $T$, roughly speaking).
Now, for our proof we can assume that for each $\sigma_i$, $\forall E \in \sigma_i(AF)$, $p(AF, AF', E, \sigma_i)$ holds true.
Then, because our semantics $\sigma_i$ satisfy the relaxed monotony principle $RMP_p$, it must hold true that $\forall E \in \sigma_i(AF)$, there exists $E' \in \sigma_i(AF')$ s.t. $E \subseteq E'$ and hence that for each $\sigma_i$ it must hold that $\exists E' \in \sigma_i(AF')$ s.t. $T \subseteq E'$.
From this it follows, again given the degrees of agreement definitions (Definition~\ref{degs-aas}), that $\deg^{h}_{min}(AAS') = \deg^{h}_{mean}(AAS') = \deg^{h}_{med}(AAS') = 1$ (our agents can still fully agree on $T$), which completes the proof.
\end{proof}
%

\minmindeg*
%
\begin{proof}
    Clearly, $\deg^{h}_{min}(AAS)$ is minimal iff $\exists \sigma, \sigma' \in SIG$ s.t. $\forall E \in \sigma(AF)$ it holds that $E \cap T = \emptyset$ and $\forall E' \in \sigma(AF')$ it holds that  $E \cap T = T$.
    Then, $\deg^{h}_{min}(AAS) = \frac{1}{2}$ if $|T|$ is even, because $\phi_{\sigma}^{sim}(AF, T, S) = \phi_{\sigma'}^{sim}(AF, T, S) = \frac{1}{2}$, given a set $S$ s.t. $S \subset T$, $|S| = \frac{|T|}{2}$ and for any set $S'$  s.t. $|S' \cap T| \neq \frac{|T|}{2}$ it must hold that $\phi_{\sigma}^{sim}(AF, T, S') \leq \frac{1}{2}$ or $\phi_{\sigma'}^{sim}(AF, T, S') \leq \frac{1}{2}$.
    If $|T|$ is odd, the tight upper bound is lower, as then we have $\deg^{h}_{min}(AAS) = \frac{\lfloor | T |  / 2 \rfloor}{|T|}$, because given the same constraints on our set $S$, if $\phi_{\sigma}^{sim}(AF, T, S) = \frac{1}{2}$ then $\phi_{\sigma'}^{sim}(AF, T, S) = \frac{\lfloor |T| / 2 \rfloor}{|T|}$ and vice versa.
    
\end{proof}

\minmindegchange*
%
\begin{proof}
From Lemma~\ref{lemma:min-min-deg}, we know that for our argumentation-based agreement scenarios $AAS$ and $AAS'$, the tight lower bounds of the degrees of minimal agreement are $\frac{\lfloor |T| / 2 \rfloor}{|T|}$ and $\frac{\lfloor |T'| / 2 \rfloor}{|T'|}$, respectively.
The maxima are obviously $1$.
Because we assume $deg_{min}^h(AAS) \geq deg_{min}^h(AAS')$, it is clear that to maximize $\Delta_{deg_{min}^h}(AAS, AAS')$ we must set $deg_{min}^h(AAS) = 1$ and minimize $deg_{min}^h(AAS')$ by setting it to $\frac{\lfloor |T'| / 2 \rfloor}{|T'|}$.
Thus, we achieve our tight upper bound of $\Delta_{deg_{min}^h}(AAS, AAS') = 1 - \frac{\lfloor |T'| / 2 \rfloor}{|T'|}$.
\end{proof}

\propaasrelaxedtwoa*
\begin{proof}
From Lemma~\ref{lemma:min-min-deg-change} we know that if $deg_{min}^h(AAS) \geq deg_{min}^h(AAS')$ then the tight lower bound of $\Delta_{deg_{min}^h}(AAS, AAS')$ is $1 - \frac{\lfloor |T'| / 2 \rfloor}{|T'|}$. 

Now, because $AAS \preceq_N AAS'$ and hence $T \subseteq T'$ and for each $\sigma_i$, $0 \leq i \leq n$, $\forall E \in \sigma_i(AF)$, $p(AF, AF', E, \sigma_i)$ holds true holds, we know that for $S = T \cap E_\cap|$ it must hold that $S \subseteq T' \cap E'_\cap$, where where $E'_\cap = \{E'_0 \cap ... \cap E'_n | E'_0 \in \sigma_0(AF'), ..., E'_n \in \sigma_n(AF'), \nexists  E''_0 \in \sigma_0(AF'), ..., E''_n \in \sigma_n(AF') \text{ s.t. } |E'_0 \cap ... \cap E'_n| < |E''_0 \cap ... \cap E''_n|\})$. Intuitively, we can say that \emph{agents continue to agree about $T \cap E_\cap$}.
Hence, we conclude that the tight upper bound is $1 - \frac{\lfloor |T'| / 2 \rfloor + |T \cap E_\cap|}{|T'|}$.
\end{proof}
Before we prove Proposition 4.1, we split the proposition into two parts, Proposition~\ref{prelim-prop-a} and Proposition~\ref{prelim-prop-b}.
\begin{customprop}{4.1.a}\label{prelim-prop-a}
Let $VAAS = (VAF, T, \sigma), VAF = (AR, AT, V, val, {\cal P })$, ${\cal P } = \langle P_0, ..., P_n \rangle$ and $VAAS' = (VAF', T', \sigma'), VAF' = (AR', AT', V', val', {\cal P }')$, $\langle P'_0, ..., P'_n \rangle$ be value-based AAS, such that $VAAS \preceq_N VAAS'$. It holds true that there exist two argumentation-based agreement scenarios $AAS = (AF, T, SIG)$ and $AAS' = (AF', T', SIG)$, such that $vdeg^{sim}_{min}(VAAS) = deg^{sim}_{min}(AAS)$ and $vdeg^{sim}_{min}(VAAS') = deg^{sim}_{min}(AAS')$, $AAS \preceq_N AAS'$.
\end{customprop}
\begin{proof}
\phantom{e}
\begin{enumerate}
    \item Let us assume, for a contradiction, that there do not exist two argumentation-based agreement scenarios $AAS = (AF, T, SIG)$ and $AAS' = (AF', T', SIG)$, such that $AAS \preceq_N AAS'$, $vdeg^{sim}_{min}(VAAS) = deg^{sim}_{min}(AAS)$ and $vdeg^{sim}_{min}(VAAS') = deg^{sim}_{min}(AAS')$.
    \item By definition of $vdeg^{sim}_{min}$, it holds true that there exist two argumentation-based agreement scenarios $AAS^* = (AF, T, SIG^*)$, $AAS^{**} = (AF', T', SIG^{**})$, such that:
    \begin{itemize}
        \item $vdeg^{sim}_{min}(VAAS) = deg^{sim}_{min}(vaas(VAAS))$ and $deg^{sim}_{min}(vaas(VAAS)) = deg^{sim}_{min}(AAS^*)$;
        \item $vdeg^{sim}_{min}(VAAS') = deg^{sim}_{min}(vaas(VAAS'))$ and $deg^{sim}_{min}(vaas(VAAS')) = deg^{sim}_{min}(AAS^{**})$;
        \item $SIG^* = \langle \sigma^*_0, ..., \sigma^*_n \rangle$, $\sigma^*_i(AF) = \sigma^*(AF_{P_i})$, $SIG^* = \langle \sigma^{**}_0, ..., \sigma^{**}_n \rangle$, and $\sigma^{**}_i(AF) = \sigma^{**}(AF_{P_i})$.
    \end{itemize}
    \item It follows from 2. that there exist two argumentation-based agreement scenarios $AAS$ and $AAS'$, s.t. $vdeg^{sim}_{min}(VAAS) = deg^{sim}_{min}(AAS)$, $vdeg^{sim}_{min}(VAAS') = deg^{sim}_{min}(AAS')$, such that $AAS = (AF, T, SIG)$, $AAS' = (AF', T', SIG')$, $SIG = SIG'$, $SIG = \langle \sigma_0, ..., \sigma_n \rangle$, $\sigma_i(AF) = \sigma(AF_{P_i})$ and $\sigma_i(AF') = \sigma(AF'_{P'_i})$; hence $AAS \preceq_N AAS'$.
    This contradicts 1. and proves the proposition.
\end{enumerate}
\end{proof}
\begin{customprop}{4.1.b}\label{prelim-prop-b}
Let $VAAS = (VAF, T, \sigma), VAF = (AR, AT, V, val, {\cal P })$, ${\cal P } = \langle P_0, ..., P_n \rangle$ and $VAAS' = (VAF', T', \sigma'), VAF' = (AR', AT', V', val', {\cal P }')$, $\langle P'_0, ..., P'_n \rangle$ be value-based AAS, such that $VAAS \preceq_N VAAS'$. $\sigma$ satisfies weak cautious monotony. There exist two argumentation-based agreement scenarios $AAS = (AF, T, SIG)$, $AAS' = (AF', T', SIG)$, $AAS \preceq_N AAS'$, such that $vdeg^{sim}_{min}(VAAS) = deg^{sim}_{min}(AAS)$ and \\ $vdeg^{sim}_{min}(VAAS') = deg^{sim}_{min}(AAS')$, $SIG = \langle \sigma_0, ..., \sigma_n \rangle$, and each $\sigma_i, 0 \leq i \leq n$ is an argumentation semantics that satisfies weak cautious monotony.
\end{customprop}
\phantom{e}
\begin{proof}
\phantom{e}
\begin{enumerate}
    \item By definition of $vdeg^{sim}_{min}$, it holds true that there exist two argumentation-based agreement scenarios $AAS^* = (AF, T, SIG^*)$, $AAS^{**} = (AF', T', SIG^{**})$, s.t. $vdeg^{sim}_{min}(VAAS) = deg^{sim}_{min}(AAS^*)$ and $vdeg^{sim}_{min}(VAAS') = deg^{sim}_{min}(AAS^{**}))$, $SIG^* = \langle \sigma^*_0, ..., \sigma^*_m \rangle$, $\sigma^*_i(AF) = \sigma^*(AF_{P_i})$, $SIG^* = \langle \sigma^{**}_0, ..., \sigma^{**}_m \rangle$, and $\sigma^{**}_i(AF) = \sigma^{**}(AF_{P_i})$.
    \item From Proposition~\ref{prelim-prop-a} it follows that we can assume $AAS^* \preceq_N AAS^{**}$.
    \item By definition of $vaas$ it holds true for $\sigma^{*}_j$, $\sigma^{**}_j$, $0 \leq j \leq m$ that $\sigma^{*}_j(AF) = \sigma(AF_{P_j}) = \sigma^{**}_j(AF)$ and $\sigma^{**}_j(AF') = \sigma(AF'_{P_j}) = \sigma^{*}_j(AF')$.
    \item Let us observe that by definition of a value-based argumentation framework $VAF^* = (AR^*, AT^*, V^*, val^*, {\cal P }^*)$, ${\cal P }^* = \{P^*_0, ..., P^*_y\}$ it holds true that for each $P^*_x, 0 \leq x \leq y$, $AT^*_{P^*_x} \subseteq AT^*$. 
    \item It follows from  4. that by definition of a subjective argumentation framework in value-based argumentation for every two value-based argumentation frameworks $VAF = (AR, AT, V, val, {\cal P })$, $VAF' = (AR', AT', V', val', {\cal P }')$, s.t. $VAF \preceq_N VAF'$, ${\cal P } = \langle P_0, ..., P_l \rangle, {\cal P } = \langle P_0, ..., P_l \rangle$ and every set of arguments $S \subseteq AR$, it holds true for that if $AR' \setminus AR$ does not attack $S$ in $AF'$, then for all $AF_{P_k} = (AR, AT_{P_k}), AF'_{P'_k} = (AR', AT'_{P'_k}), 0 \leq k \leq l$, it holds true that $AR' \setminus AR$ does not attack $S$ in $AF'_{P'_k}$.
    \item From 2. and 4. it follows that by definition of weak cautious monotony, if $\sigma$ satisfies weak cautious monotony, then each $\sigma_i, 0 \leq i \leq n$ is an argumentation semantics that satisfies weak cautious monotony.
    This proves the proposition.
\end{enumerate}
\end{proof}
\noindent Let us again state Proposition~\ref{prelim-prop} of the paper.
\prelimprop*

\begin{proof}
The proof is direct by Propositions~\ref{prelim-prop-a}~and~\ref{prelim-prop-b}.
\end{proof}

%
\vaaspropone*
%
\begin{proof}
By definition of $vdeg_{min}$, it holds true that $vdeg^{h}_{min}(VAAS) = deg^{h}_{min}(vaas(VAAS))$ and $vdeg^{h}_{min}(VAAS') = deg^{h}_{min}(vaas(VAAS'))$.
From Proposition~\ref{prelim-prop} it follows that for every two value-based AAS $VAAS, VAAS'$,  such that $VAAS \preceq_N VAAS'$, it holds true that there exist argumentation-based agreement scenarios $AAS = vaas(VAAS), AAS' = vaas(VAAS')$, such that $AAS = (AF, T, SIG)$ and $AAS' = (AF', T', SIG')$, $AAS \preceq_N AAS'$, $SIG = SIG' = \langle \sigma_0, ..., \sigma_n \rangle$ and for every $\sigma_i, 0 \leq i \leq n$, $\forall E \in \sigma_i(AF)$, it holds true that $cm(AF_{P_i}, AF'_{P_i}, E, \sigma_i)$ holds true. Hence, the proof follows from Proposition~\ref{proposition-3a}.
\end{proof}
%
\vaasproptwo*
%
\begin{proof}
By definition of $vdeg^{h}_{min}$, it holds true that $vdeg^{h}_{min}(VAAS) = deg^{h}_{min}(vaas(VAAS))$ and $vdeg^{h}_{min}(VAAS') = deg^{h}_{min}(vaas(VAAS'))$.
From Proposition~\ref{prelim-prop} it follows that for every two value-based AAS $VAAS, VAAS'$,  such that $VAAS \preceq_N VAAS'$, it holds true that there exist argumentation-based agreement scenarios $AAS = vaas(VAAS), AAS' = vaas(VAAS')$, such that $AAS = (AR, T, SIG)$ and $AAS' = (AF', T', SIG')$, $AAS \preceq_N AAS'$, $SIG = SIG' = \langle \sigma_0, ..., \sigma_n \rangle$ and for every $\sigma_i, 0 \leq i \leq n$, $\forall E \in \sigma_i(AF)$, it holds true that $cm(AF_{P_i}, AF'_{P_i}, E, \sigma_i)$ holds true. Hence, the proof follows from Proposition~\ref{proposition-4a}.
\end{proof}
\end{appendices}



\end{document}